%% file: main.tex
\documentclass[10pt,twocolumn,letterpaper]{article}

\usepackage[pagenumbers]{cvpr} 

\input{preamble}

\definecolor{cvprblue}{rgb}{0.21,0.49,0.74}
\usepackage[pagebackref,breaklinks,colorlinks,citecolor=cvprblue]{hyperref}

\usepackage{amssymb}
\usepackage[linesnumbered,ruled,vlined]{algorithm2e}

\usepackage{float} 
\usepackage{pgfplots}
\usepackage{pgfplotstable}
\usepackage{subcaption}
\usepackage{tikz}
\usepackage{multirow}
\usepackage{amssymb}
\usepackage{amsthm}
\newtheorem{thm}{Theorem}
\newtheorem{defn}{Definition}

\def\paperID{15579} 
\def\confName{CVPR}
\def\confYear{2024}

\title{On Transfer in Classification: How Well do Subsets of Classes Generalize?}

\SetCommentSty{mycommfont}

\SetKwInput{KwInput}{Input}
\SetKwInput{KwOutput}{Output}

\author{Raphael Baena, Lucas Drumetz, Vincent Gripon \\
IMT Atlantique, Lab-STICC, UMR CNRS 6285, F-29238, France\\
{\tt\small firstname.surname@imt-atlantique.fr}
}

\begin{document}
\maketitle
\begin{abstract}
          In classification, it is usual to observe that models trained on a given set of classes can generalize to previously unseen ones, suggesting the ability to learn beyond the initial task. This ability is often leveraged in the context of transfer learning where a pretrained model can be used to process new classes, with or without fine tuning. Surprisingly, there are a few papers looking at the theoretical roots beyond this phenomenon. In this work, we are interested in laying the foundations of such a theoretical framework for transferability between sets of classes. Namely, we establish a partially ordered set of subsets of classes. This tool allows to represent which subset of classes can generalize to others. In a more practical setting, we explore the ability of our framework to predict which subset of classes can lead to the best performance when testing on all of them. We also explore few-shot learning, where transfer is the golden standard. Our work contributes to better understanding of transfer mechanics and model generalization.
\end{abstract}

In recent years, \emph{transfer learning} -- the ability of a model to leverage knowledge gained from one task and apply it to another -- has gained a lot of attention, particularly with the rise of foundation models and few-shot settings. We argue that transfer can be thought of as a form of \emph{generalization} from a task to others, and as such can be studied with a mathematically sound framework. Therefore, the primary objectives of this paper are to provide a precise definition of transfer learning within the context of classification and to leverage it for investigating fundamental questions in the field, such as which classes generalize to which others.

Transfer learning has been of interest for a while~\cite{caruana1997multitask,hintonimagenet,transfersurvey}, starting with domain adapation~\cite{domainadaptation}. But, recently its potential surged due to self-supervised learning techniques enabling versatile representations~\cite{selfsupervised,vit8} which can be employed on a large range of problems. With abundant data and computational resources, the foundation models emerged, transforming the field~\cite{transfersurvey2}. Nowadays, researchers routinely use pretrained models like CLIP~\cite{clip} and Vision Transformer VIT-8~\cite{vit8} for specific tasks.

Interestingly, in some cases, fine-tuning these models is not even necessary, as they can serve directly as effective feature extractors~\cite{hintonimagenet}. In classification setting, the direct application of metric learning on the features, makes it possible to classify new sets of classes ($C'$), expanding the classification abilities beyond the original classes the models were trained on ($C$)~\cite{featuretransfer,nncfewshot}. Few-shot learning, for example, involves tasks where the training and target domains do not overlap~\cite{fewshot,fewshotsurvey}, at least in terms of semantics.

Despite the widespread success of transfer learning, the underlying mechanisms that enable this ability remain largely unclear and, to our knowledge, insufficiently formulated. Several authors have proposed definitions for fundamental concepts such as \emph{task}, \emph{task-domain}, and the different kinds of transfer, like feature-based, instance-based, and parameter-based~\cite{transfersurvey,transfersurvey2}. Other works have focused on deriving theoretical bounds on convergence~\cite{theorytransferactive,theorytransfertaskdiversity}.

\begin{figure*}
 \centering

 \begin{subfigure}{0.45\textwidth}
  \centering
  \begin{tikzpicture}[
  scale =0.8,
  node distance=3cm,
  transform shape,
  state/.style={circle, draw, minimum size=2cm, text width=2cm, align=center},
  bluestate/.style={state, fill=blue!40},
  redstate/.style={state, fill=red!40},
  greenstate/.style={state, fill=green!40},
  >=latex,
  ]
  \node[state] (dog_truck) {dog/truck};
  \node[bluestate, right of=dog_truck] (cat_dog) {cat/dog};
  \node[redstate, left of=dog_truck] (auto_cat) {auto/cat};
  \node[redstate, above of=auto_cat] (cat_truck) {cat/truck};
  \node[state, above of=dog_truck] (auto_dog) {auto/dog};
  \node[greenstate, right of=auto_dog] (auto_truck) {auto/truck};
  
  \draw[<->] (auto_cat) to node {} (cat_truck);
  \draw[->] (auto_cat) to node {} (auto_dog);
  \draw[->] (cat_truck) to node {} (auto_dog);
  \draw[->] (auto_cat) to node {} (dog_truck);
  \draw[->] (cat_truck) to node {} (dog_truck);
  \draw[->] (cat_dog) to node {} (dog_truck);
\end{tikzpicture}
  \label{subfig:tikz_graph_1}
 \end{subfigure}
 \hfill
 \begin{subfigure}{0.45\textwidth}
  \centering
\begin{tikzpicture}[
  scale =0.8,
  node distance=3cm,
  transform shape,
  state/.style={circle, draw, minimum size=2cm, text width=2cm, align=center},
  redstate/.style={state, fill=red!40},
  bluestate/.style={state, fill=blue!40}, 
  greenstate/.style={state, fill=green!40},
  >=latex,
  ]
  \node[redstate] (shirt_bag) {shirt/bag};
  \node[state, right of=shirt_bag] (shirt_trouser) {shirt/trouser};
  \node[state, left of=shirt_bag] (coat_trouser) {coat/trouser};
  \node[bluestate, above of=coat_trouser] (shirt_coat) {shirt/coat};
  \node[redstate, above of=shirt_bag] (coat_bag) {coat/bag};
  \node[greenstate, right of=coat_bag] (trouser_bag) {trouser/bag};

  \draw[<->] (coat_bag) to node {} (shirt_bag);
  \draw[->] (coat_bag) to node {} (shirt_trouser);
  \draw[->] (shirt_bag) to node {} (shirt_trouser);
  \draw[->] (coat_bag) to node {} (coat_trouser);
  \draw[->] (shirt_bag) to node {} (coat_trouser);
\end{tikzpicture}
  \label{subfig:tikz_graph_2}
 \end{subfigure}
 
 \caption{Hasse Diagrams illustrating the order relationship of models learned on pairs of classes. Each model is characterized by the pairs of classes it can successfully separate. Arrows from one model $A$ to another $B$ indicate that $A$ is more expressive and can separate all the pairs that $B$ can separate. The colored models are the most expressive and are called fundamental. Models sharing the same colors are considered equivalent as they can separate the same pairs. Uncolored models could be disregarded as the corresponding pairs can be separated with another model. On the left, we consider 4 classes from CIFAR-10, and the diagram shows that `cat/truck' and `auto/cat' are equivalent and more expressive than `auto/dog' and `dog/truck'. In other words, models `auto/dog' and `dog/truck' can be disregarded as they offer no additional separability compared to `cat/truck' and `auto/car. On the right, considering 4 classes from FASHION-MNIST, the diagram demonstrates that `coat/bag' and `shirt/bag' are equivalent and more expressive than `coat/trouser' and `shirt/trouser'. }
 \label{fig:hasse}
\end{figure*}
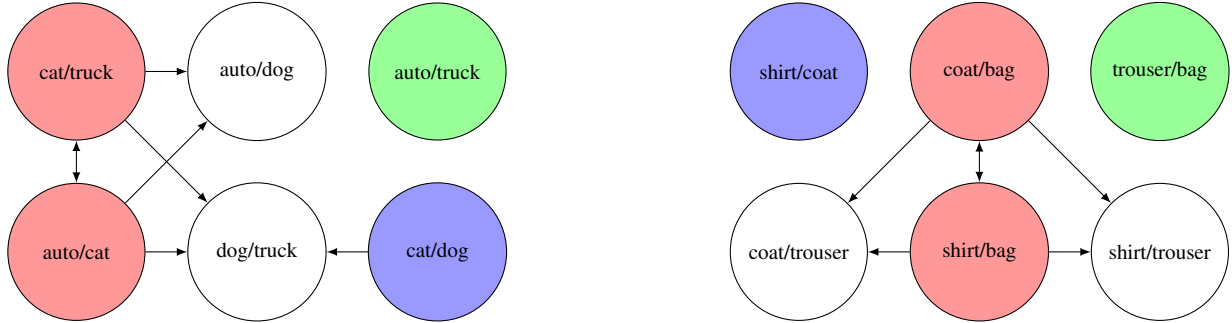

While these existing approaches offer valuable insights, our work complements them by introducing new perspectives with simple yet meaningful definitions. We propose a novel theoretical model that establishes an order relationship of learning, conveniently represented by a Hasse Diagram as shown in Figure~\ref{fig:hasse}. We present a comprehensive theoretical framework, and metric called $separability$, for characterizing the transferability of models trained on one set of classes ($C$) to another set of classes ($C'$).
This framework serves as a tool for investigating fundamental questions related to learning, such as:\\

\begin{enumerate}
\item Can we identify informative indicators of how well a model trained on one set of classes ($C$) will perform on another set of classes ($C'$)?
\item Which classes play the most crucial role in the learning process when fine-tuning a model? Which set of classes leads to the best generalization and transferability?
\end{enumerate}

To address these fundamental questions and demonstrate the utility of our approach, we conduct experiments in various scenarios, offering empirical evidence and insights into the transferability of deep learning models.

\section{Related Work}
\paragraph{Foundation Models:}
The remarkable revival of neural network-based machine learning around 2010~\cite{fondation}, fueled by large datasets and architectures with substantial parameters~\cite{deeplearningnature,imagenet}, led to the realization that these trained models could be easily adapted to various domains and classes~\cite{imagenet}. Subsequently, these pre-trained models evolved into even larger foundation models~\cite{fondation}, trained on surrogate tasks and capable of effective knowledge transfer across diverse tasks and domains~\cite{CLIPfewshot,objectdetection,denseclip}. This paradigm shift, known as homogenization, encouraged most researchers to adopt the same pre-trained models, such as BERT~\cite{bert}, GPT~\cite{GPT}, and CLIP~\cite{clip}, which are now standard benchmarks in the field. 

On the paper, foundation models can indeed be thought of as a surrogate for transfer learning: when learning on a large, representative universal dataset, there is no more need to adapt to any different domain. Yet, recent papers have shown the merits of transfer learning even in the case where the downstream task is a fine-grain refinement of the initial domain~\cite{baena2022preserving,hierchicalloss}. These recent advancements drive us to create a practical theoretical framework to investigate and understand transferability in classification settings.

\paragraph{Transfer Learning:} At a technical level, the generalization ability of foundation models relies on transfer learning~\cite{caruana1997multitask,transferimagenet}.
First, the models are trained to solve specific tasks like classification, using labeled datasets that represent particular data domains. Then the challenge lies in extending their capabilities to tackle other tasks or adapt to different domains. Initially, transfer learning was limited to domain adaptation~\cite{caruana1997multitask}. However, with the advent of deep learning, the concept of transfer has been extended to encompass changes in both the task and the data domain~\cite{transfersurvey,transfersurvey2,selfsupervised}. Transfer learning aims to leverage knowledge obtained from a source task and dataset, even if they are not strongly related to the target task or dataset, in order to reduce the learning cost and solve the target task. It is particularly useful when acquiring a large amount of labeled data is costly/impossible, or to avoid overfitting.

Transfer learning can be categorized into three main types: transductive, inductive, and unsupervised~\cite{transfersurvey,transfersurvey2}. Transductive transfer learning occurs when only the source data is labeled, while inductive transfer learning involves both the source and target data being labeled. Unsupervised deep transfer learning involves unlabeled data.

Furthermore, transfer learning can be decomposed into four approaches according to~\cite{transfersurvey,transfersurvey2}, instance-based (weighting samples)\cite{instanceweighting} , feature-based or mapping-based~\cite{featuretransfer}, parameter-based or model-based~\cite{parameters}, and relational-based or adversarial-based approaches~\cite{davis2009deep}.

Our paper focuses on feature-based methods. Indeed, recent research demonstrates that deep learning models can acquire highly generalized representations, which can be directly employed as feature extractors to solve diverse tasks without requiring fine-tuning \cite{featuretransfer,nncfewshot,selfsupervised}.

\paragraph{Few-Shot Learning:} Few-Shot Learning~\cite{fewshot} is a highly relevant research area closely tied to the theoretical framework proposed in this paper. The goal is to learn to recognize new classes (\textit{novel classes}) with only a few labeled examples (shots) per class, even when the model was originally trained on a different set of classes (\textit{base classe}s). This corresponds to a realistic setting where obtaining a large amount of labeled data for each new class might be impractical. Few-shot learning can be approached in various ways, and one common methods are feature-based~\cite{fewshotsurvey}. The model, initially trained on \textit{base classes}, serves as a feature extractor for \textit{novel class} samples. A new classifier, often a nearest mean classifier (NCM)~\cite{nncfewshot}, is then applied on these features to predict novel classes.

\paragraph{Exploring Transferability:}The emergence of Self-Supervised Learning techniques has played a crucial role in achieving good transferability of features~\cite{selfsupervised,selfsupervisedsurvey}. Self-Supervised Learning (SSL) involves training models on pretext tasks that do not require human-labeled annotations, allowing the models to learn meaningful representations of the data. Therefore, in this paper, we delve into the investigation of transfer learning using a Vision Transformer Dino VIT-8~\cite{vit8} pretrained on a large dataset using SSL method, which is a standard benchmark in the field.

Advancements in transfer learning have driven several authors to investigate fundamental questions within the domain. In~\cite{taskatanov}, the authors propose a method called \textit{Task Discovery} to uncover tasks (subsets of classes) on which a trained network can generalize well. This method involves computing similarity scores where higher similarity suggest that networks are likely to converge to stable solutions. The researchers demonstrate that using human labels leads to better generalization compared to random labels, enabling the identification of tasks or sets of classes that result in optimal generalization in term of error rate. However, their approach differs from our work, as we provide a theoretical framework, whereas they focus on identifying subsets of classes leading to stable solutions.

Noticeable theoretical works in transfer learning include~\cite{transfersurvey,transfersurvey2}, where the authors categorize different transfer learning approaches and provide definitions for fundamental concepts such as tasks and data domains. Theoretical bounds on convergence rates for specific cases, such as active learning, have also been derived ~\cite{theorytransferactive,theorytransfertaskdiversity}.

\section{Proposed Theoretical Framework}
In this section, we propose a novel theoretical model in the context of classification, where the primary objective is to separate each class from others. First, we consider that classes are separated by models that act as bi-partitions. For instance, a hyperplane can effectively split a space into two regions, where ideally, one class resides on one side, and the other class on the opposite side. While modern deep networks often use a one-vs-all approach for class separation, we take a more fundamental perspective: we consider classes to be separated only when each pair of classes is effectively separated. In essence, we view the problem of classification as a partitioning problem, where for each pair of classes, there must be at least one partition in which one class belongs to one set and the other class to the other set. By adopting this viewpoint, it appears that some bipartitions, or models, are more expressive than others. This leads us to introduce an ordering relationship among models.

\subsection{General Case}Let us consider a space $E$ and a set of $n$ disjoints subsets of $E$ denoted $C$. The elements of $C$ are called classes. We introduce a model as a bipartition of the space $E$, separating two classes when one class belongs to one partition and the other class belongs to the other partition.
\begin{defn}[Model]
 A model $M = \{M^1, M^2\} $ is a bipartition of the space $E$, i.e, $E = M^2 = E\setminus M^1$.

Let us say that $M$ separates $I,J \in C$, if: ($I \cap M^1 = \emptyset) \land (J \cap M^2 = \emptyset)$ and we note $I \perp_M J$.
\end{defn}

In practice, the models are inferred from data. Here, we assume that for each pair of classes, we learn a specific model that separates that pair; otherwise, the problem would not be separable. Therefore, each model is associated with a specific pair of classes on which it was learned. 

\begin{defn}[Model associated with a pair of classes]
Let us denote $M_{I,J}$ a model that separates $I$ and $J$.
\end{defn}
 The way we obtain such a model is voluntarily not explicit right now but will be clarified in the methodology section. It is possible that some models separate more than the pair with which they are defined. Therefore, we introduce for each model the set of all pairs it separates.
\begin{defn}[Set of separable pairs]
 Let us note $\mathcal{C}(M)$ the set of pairs of elements in $C$ that are separated by $M$.
\end{defn}

From this definition, it is evident that some models are more expressive than others. For instance, a model $M$ may separate all the pairs that a model $M'$ separates, which naturally leads us to define the following order relationship.
 
\begin{defn}[Ordering models]
 Let us consider two models $M,M'$. The model $M$ is said to be less expressive than the model $M'$ if: $\mathcal{C}(M) \subseteq \mathcal{C}(M')$, and we use the notation $M \leq M'$. The relation $ \leq $ is an order relationship.
\end{defn}
Some models separate exactly the same set of pairs and are said to be equivalent. We use the following notation.
 \begin{defn}[Equivalent models]
 Two models $M,M'$ are equivalent if $M \leq M'$ and $M' \leq M$, denoted $M \equiv M'$.
\end{defn}
The previous order relationship can be conveniently represented using Hasse Diagrams, providing a visual representation of the partial order among models. In the Hasse Diagram, each model is represented as a node, and an arrow from $M'$ to $M$ indicates that $M \leq M'$. The Hasse Diagram allows us to identify the most expressive models at the top and the least expressive models at the bottom. Less expressive models can be disregarded since the pairs they separate can be distinguished using other models. 

Figure~\ref{fig:hasse} depicts the Hasse diagrams of two different datasets (restricted to 4 classes). Such diagrams were obtained by following our methodology section. More precisely, we computed one model for each pair of classes. Then for each model we determined the set of pairs they separate (See matrix $\mathbf{S}$ in Algorithm~1).

Now, we are interested in the most expressive models, those that cannot be learned otherwise except through an equivalent model. The pairs on which they are learned are called fundamental since they cannot be removed. They correspond to the maximal elements of the order.

\begin{defn}[Fundamental pair]
 A pair of classes $\{I,J\}$ is said to be fundamental if $I \perp_{M_{K,L}} J \Rightarrow M_{I,J} \equiv M_{K,L}$.
\end{defn}
Of particular interest is the characterization of the minimum number of pairs needed to separate all the classes. This is termed the fundamental number, as fewer pairs cannot suffice to learn all essential models.
\begin{defn}[Fundamental number]
The fundamental number $\mathcal{F}(C)$ is defined as the minimum cardinal of a set $S$ of models such that $\forall I,J \in C, \exists M \in S, I \perp_M J$. Note that by definition we have an upper bound: $\mathcal{F}(C) \leq \binom{n}{2}$.
\end{defn}
Instead of characterizing the fundamental number by pairs, we can use models as proved by the following theorem. It is a way to see the fundamental number as the minimal number of models needed to separate all the pairs. 
\begin{thm}\label{thm:fundamentalpairs}
Consider the set of all fundamental pairs. Let us remove pairs until this set contains only the ones whose associated models are not equivalent to one another. Then the cardinality of this set is noted $\mathcal{F}(C)$.
\end{thm}

Definition. 7 showed there is an upper bound to $\mathcal{F}(C)$. It turns out $\mathcal{F}(C)$ can also be lower bounded. Notably, both of these bounds can be attained by at least one example.
\begin{thm}
 Let $n \geq 2$. Then the fundamental number $\mathcal{F}(C)$ is bounded as follows: 
 \begin{align*}
  \log_2(n) \leq \mathcal{F}(C) \leq \binom{n}{2}
 \end{align*}
 Both bounds are achieved by at least one example.
\end{thm}

\begin{proof}
 Let us proof the lower bound. We consider $n \geq 2$, we want to distribute the pairs of $C$ on a minimal number of models $\mathcal{F}(C)$. First, we note that in any case, the minimum can be achieved when all elements of $C$ belong to one partition of each model: if one element $X$ does not belong to any partition of a model $M$ it can be added without consequences to one of the partitions $M^1,M^2$.

 By doing so, the models describe all elements of $C$ which can be represented by a binary encoding: $0$ for element included in $M^1$ and $1$ for element included in $M^2$.  The problem is then equivalent to the encoding of a source alphabet $C$ into a decodable code over an binary alphabet. Then according to the Kraft-McMillan inequality, the following bound holds: $\mathcal{F}(C) \geq \log_2(n)$ \cite{kraft1949device}.
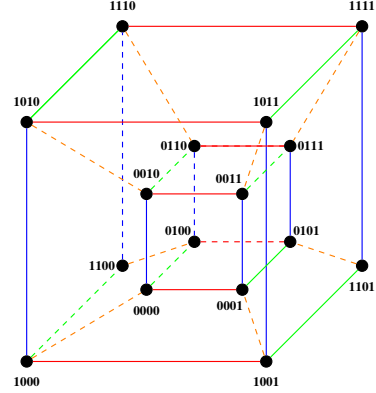
\begin{figure}
 \centering
\resizebox{0.28\textwidth}{!}{
\begin{tikzpicture}[
 line width=0.6pt,
 every node/.style={circle, draw, fill, minimum size=6pt, inner sep=0pt, font=\scriptsize\bfseries}]

 \pgfsetxvec{\pgfpoint{0.9cm}{0.0cm}}
 \pgfsetyvec{\pgfpoint{0.0cm}{0.9cm}}

 \foreach \point / \id / \angle in {
  (0,0)/1000/270,
  (0,5)/1010/90, 
  (5,0)/1001/270, 
  (5,5)/1011/90, 
  (2,2)/1100/180, 
  (2,7)/1110/90, 
  (7,2)/1101/270,
  (7,7)/1111/90, 
  (2.5,1.5)/0000/270, 
  (2.5,3.5)/0010/90, 
  (4.5,1.5)/0001/250, 
  (4.5,3.5)/0011/120, 
  (3.5,2.5)/0100/170, 
  (3.5,4.5)/0110/180, 
  (5.5,2.5)/0101/10, 
  (5.5,4.5)/0111/0} 
 {
  \node (\id) at \point [label=\angle:\id] {};
 }

 \path[draw=blue] 
 (1010) edge (1000)
 (1011) edge (1001)
 (1111) edge (1101)
 (0010) edge (0000) 
 (0011) edge (0001)
 (0101) edge (0111)
 ;
 \path[dashed,draw=blue]
 (1100) edge (1110)
 (0110) edge (0100);
 
 \path[draw=red] 
 (0010) edge (0011)
 (0000) edge (0001)
 (0110) edge (0111)
 (1010) edge (1011)
 (1000) edge (1001)
 (1110) edge (1111);
 \path[dashed,draw=red]
 (0110) edge (0111)
 (0100) edge (0101);

 \path[draw=green] 
 (0001) edge (0101)
 (1001) edge (1101)
 (1010) edge (1110)
 (1010) edge (1110)
 (1011) edge (1111);
 \path[dashed,draw=green]
 (1000) edge (1100)
 (0000) edge (0100)
 (0010) edge (0110)
 (0011) edge (0111);

 \path[dashed,draw=orange]
 (1000) edge (0000)
 (1001) edge (0001)
 (1101) edge (0101)
 (1111) edge (0111)
 (1100) edge (0100)
 (1010) edge (0010)
 (1110) edge (0110)
 (1011) edge (0011);
 
\end{tikzpicture}}
 \caption{Encoding of 16 classes with a minimal number of bi-partitions (models): $log_2(16) = 4$. Pairs/edges which share the same color correspond to equivalent models. }
 \label{fig:hypercube}
\end{figure}

\end{proof}
In the appendix, we provide one example that reaches the upper bound and another example that reaches the lower bound. Notably, the lower bound is achieved by the hypercube configuration, as illustrated in Figure~\ref{fig:hypercube}.

\subsection{Hyperplanes}
The previous definitions provided a general framework for understanding the problem of class separation, but they did not take into account the specific geometry of the classes. To make the definitions more practical and applicable to real-world scenarios, we propose a refinement. 

Consider a space $E = \mathbb{R}^d$ of dimension $d$, and $n$ classes defined as independent $\mathbf{C_i}$ in $E$, each following its own probability density function $\mathrm{p}_i$ and $\varepsilon > 0$ be a scalar value. Now, we will restrict to models that consist of affine hyperplanes with an error tolerance of $\varepsilon$ (as defined in Definition~\ref{def:sep} below). This refinement allows us to introduce a level of practicality in our model by considering data points in a continuous space without having to restrict to separable classes. The use of affine hyperplanes with a controlled error $\varepsilon$ ensures that our models can handle some level of noise and inexactness while separating the classes effectively. 

\begin{defn}[Affine hyperplane]
 An affine hyperplane $h$ is an subset of $\mathbb{R}^n$ described with a single affine equation: $
  \langle \mathbf{w}, \mathbf{x} \rangle + b = 0 $ where $\mathbf{w}$ is a nonzero vector, and $b$ a scalar.\\
 Such a hyperplane defines a model $H$ of $E$ with the associated bipartion $M^1 = \{x \in E \vert \langle \mathbf{w}, \mathbf{x} \rangle + b \geq 0\}$ and $M^2 = \{x \in E \vert \langle \mathbf{w}, \mathbf{x} \rangle + b <0 \}.$
\end{defn}

In~\cite{lowerbound}, the authors elegantly proved that $n$ points in $\mathbb{R}^d$ can be separated by at most $\lceil (n-2^{\lceil \log(d)\rceil})/d + \lceil \log(d) \rceil\rceil$ hyperplanes. Using the same definitions than before we have a bound on $\mathcal{F}(C)$ in the specific case where classes are single points. The results can be easily extended to unit balls; however, when dealing with non-convex geometries the interface between the sets can become arbitrarily complex, making it challenging to establish meaningful bounds on the number of hyperplanes required as noted in~\cite{geometrybound}.

When dealing with random variables, new definitions are necessary to define separability, as in real cases, hyperplanes may not be capable of perfectly separating classes.

\begin{defn}[Epsilon-Separability] Let us consider two random variables $\mathbf{C_i} , \mathbf{C_j}$, and a hyperplane $H$. Without loss of generality, let us assume that $\mathrm{P}\left[ \langle \mathbf{w}, \mathbf{C_i} \rangle + b \geq 0 \right] \geq 0.5$. We define the probability error of the hyperplane $H$ as:
\begin{align*}
\mathrm{P}_{err}(H,\mathbf{C_i},\mathbf{C_j}) &= \frac{1}{2} \left( \mathrm{P}_i\left[ \langle \mathbf{w}, \mathbf{C_i} \rangle + b < 0 \right] \right. \\
&\left.+ \mathrm{P}_j\left[ \langle \mathbf{w}, \mathbf{C_j} \rangle + b \geq 0 \right] \right)
\end{align*}
 $H$ separates $\mathbf{C_i},\mathbf{C_j}$ if: $\mathrm{P}_{err} < \varepsilon$ and we note $ \mathbf{C_i} \perp_H \mathbf{C_j}.$
\label{def:sep}
\end{defn}

Let us apply this definition to two real-valued random variables following a normal distribution with unit variance: $\mathbf{\mathbf{C_i}} \sim \mathcal{N}(\mu_i ,1), \ C_j \sim \mathcal{N}(\mu_j ,1)$. Assuming $\mu_j > \mu_i$, the probability error of a hyperplane $h$ is:
\begin{align*}
 &\mathrm{P}_{err}(h,i,j) = \mathrm{P}_i\left[ C_i < - b/w \right] + 1 - \mathrm{P}_j\left[ C_j \leq -b/w \right] \\
 &= \frac{1}{4}\left( 2 +\left(erf\left( -\frac{b/w + \mu_2}{\sqrt{2}} \right) -erf\left( -\frac{b/w + \mu_1}{\sqrt{2}} \right) \right)\right).
\end{align*}

Note that the error is minimal when: $b/w + \mu_1 = b/w - \mu_2$ i.e $ b/w = (\mu_1 + \mu_2)/2$ \cite{bishop2006pattern}.\\

\begin{defn}[Empirical Probability of Error]
In practical scenarios, exact probability density functions are often unavailable, and instead, they are approximated empirically based on a given set of data points. Let us consider a hyperplane $h$, and two classes $i$ and $j$. Without loss of generality, we assume that samples $\mathbf{x}^i$
belonging to class $i$ are correctly classified if $h(\mathbf{x}^i)>0$, and conversely, samples $\mathbf{x}^j$
 belonging to class $j$ are correctly classified if $h(\mathbf{x}^i)\leq 0$. Let $\mathbf{X_i} = \{\mathbf{x}^i_1, \cdots, \mathbf{x}^i_n\}$ and $\mathbf{X_j} =\{\mathbf{x}^i_1, \cdots, \mathrm{x}^i_n\}$, represent sets of $n$ samples drawn of classes $i$ and $j$ respectively. The approximated probability error is defined as follows:
\begin{align*}
\hat{\mathrm{P}}_{err}(h,i,j) = \frac{1}{2n} \sum_{k=1}^n \left( \mathbf{1}_{{h(x_k^i) < 0}} + \mathbf{1}_{{h(x_k^j) \geq 0}} \right)
\end{align*}
\end{defn}
\section{Methodology}
When trained on a set of classes $C$, a deep learning architecture can learn to distinguish more classes than those present in $C$. Given a new set of classes $C'$, we then define the \textbf{separability} as the number of distinct pairs of $C'$ that the trained architecture separates. In this experimental part, we investigate whether the \textbf{separability} of pretrained architectures, is a relevant indicator of transferability.

\paragraph{Settings:} Let $\varepsilon > 0$ be a scalar, $C$ be a set of $N$ classes, and $P$ be the set of all pairs of classes in $C$. As explained in the theoretical part, we consider a model for each pair of classes. In our experiments, we choose linear hyperplanes--which are typically the final layer (heads) of a neural network--inferred on features provided by a feature extractor to capture high-level information. Thus, following the common approach in classification tasks, we decompose the network into two parts: the feature extractor $f$ and the heads $h$, which are linear layers responsible for decisions.  

\paragraph{Separability:} Given a new set of classes $C'$, composed of $N'$ classes, we evaluate the separability of models trained on pairs $P$ from the set of classes $C$ as follows: we count the number of pairs in $C'$ that are separated by models trained on pairs of $C$ using Algorithm~1. For simplicity of notation we assume $n$ examples per class in the algorithm.

\begin{algorithm}[!ht]
\label{alg:algo}
\DontPrintSemicolon

\KwInput{$\mathbf{F} \in \mathbb{R}^{d_{embd}, n, N'}$: features,\\ trained models: $P = \binom{N}{2}$ $heads$, $\varepsilon$: a scalar.}
\KwOutput{separability: integer between $0$ and$\binom{N'}{2}$.}

$\mathbf{D} \leftarrow \text{heads}(\mathbf{F}) \in \mathbb{R}^{n, N', P}$ \tcp*{predictions}
$idx\_pair \leftarrow 0$\;
$\mathbf{S} \leftarrow zeros(\binom{N'}{2},P) \in\mathbb{R}^{\binom{N'}{2},P} $ \;
\tcp*{iterating over the pair of $C'$}
\For{$i \leftarrow 1$ \KwTo $N'$ }{
         
	\For{$j \leftarrow 1$ \KwTo $N'$} {
 \tcp*{average signs of prediction}
		$\text{\textbf{sign}}\_i \leftarrow \text{\text{\textbf{sign}}}(\text{avg}
  (\mathbf{D}[:, i, :]),\text{dim} = 0))$\;
		$\text{\textbf{sign}}\_j \leftarrow -1 \times \text{\text{\textbf{sign}}}(\mathbf{D}[:, j, :]) \in \mathbb{R}^{P}$\;
  \tcp*{concatenates labels}
		$\mathbf{D}\_{\text{true}} \leftarrow 
  [\text{\text{\textbf{sign}}\_i} \times \text{ones}(n, 1),\text{\text{\textbf{sign}}\_j} \times \text{ones}(n, 1)] \in \mathbb{R}^{2n, P}$\;
        $\mathbf{D}_\_i\_j \leftarrow [\mathbf{D}[:, i, :], \mathbf{D}[:, j, :]$\;
        \tcp*{accuracy of each head}
		$\text{\textbf{acc}}\_i\_j \leftarrow \frac{1}{2n}\sum_k \left[\mathbf{D}_\_i\_j] == \mathbf{D}\_{\text{true}}\right][k, :] $\;
  \tcp*{separability of each head}
		$\mathbf{S}[idx\_pair, :] \leftarrow 1-
  \text{\textbf{acc}}\_i\_j < \varepsilon$\;
		$idx\_pair \leftarrow idx\_pair + 1$\;
	}
}

$\text{separability} \leftarrow \sum \max(\mathbf{S}, \text{dim} = 1)$

\caption{Compute the \textbf{Separability} of P models (heads), on a set $C'$ of N' classes.}
\end{algorithm}

\paragraph{Estimating the generalization potential of a set of classes} 
Consider two sets of classes $C$ on which a model is trained and $C'$ on which it's evaluated. To assess whether $C$ is a good set to generalize on $C'$, we use \textbf{a pretrained feature extractor} to represent the data. We then train one head per pair on this representation, assigning a binary label to each pair, with one class labeled 0 and the other labeled 1. The loss is computed using binary cross-entropy for each head. During the training process, the feature space is frozen, and only the heads' parameters are updated.

Finally, following Algorithm~1, we compute for this architecture the \textbf{separability} on $C'$, which we interpret as an indicator of transferability from $C$ to $C'$.\\

In the experiments, we employ two widely recognized pretrained models: a ResNet-50~\cite{resnet} and a Vision Transformer DINO VIT-8~\cite{vit8}, both were pretrained on the IMAGENET dataset~\cite{imagenet}. Vision Transformer models have been recognized recently for their proficiency in learning generic features through self-supervised learning, while ResNet-50 have served as a standard benchmark in transfer learning~\cite{hintonimagenet}. Moreover, in~\cite{touvronresnet,convnet} the authors have also shown that Residual Networks remain competitive despite the prevalent use of transformers in recent years.

\begin{table*}[h]
  \caption{Correlation between multiple metrics given by a pretrained network and separability when training from scratch.}
  \label{tab:my_label}
  \centering
  \begin{tabular}{lllll}
  \toprule
  Metric & CIFAR10\_4 & CIFAR10\_6 & FASHION & MNIST \\
  \midrule
  Separability & $ 0.68 $ & $0.47$ & $0.86$ & $0.51$\\
  Class Separability & $0.93$ & $0.76$& $0.92$ & $0.67$\\
  Pair separability & $0.90$ & $0.74$ & $0.95$& $0.67$ \\
 \bottomrule
  \end{tabular}
  \label{tab:coefffromscratch}
 \end{table*}

\paragraph{Evaluating the separability of trained architecture} Similar to the previous paragraph, we train a model on $C$, but now the feature extractor is no longer frozen. We either use the same pretrained feature extractor or use a new one that we train from scratch. Then we evaluate again the separability on $C'$. Hopefully, the separability given by the pretrained network and these networks will be correlated.

We anticipate that in the case of finetuning, the separabilities will serve as relevant performance indicators since the same network are used for computing the two separabilities: before and after finetuning. However, when it comes to training from scratch, the relationship between the separabilities becomes more uncertain since the feature extractor are different. It is difficult to predict whether separabilities will be reliable indicators of performance.

\section{Experiments}
We conduct experiments on both classification and few-shot learning datasets with the aim of investigating the sets of classes that lead to the best generalizations.

For classification experiments, let $C'$ denote the set of classes within a dataset. Our goal is to train or finetune networks on subsets $C$ of classes (where $C \subset C'$) and identify the subsets that enable an accurate classification of most classes in $C'$. irst, as described in the methodology, we compute the separability achieved on $C'$ by a pretrained architecture (feature extractor frozen, heads trained on pairs of $C$). Hopefully, the observed separability will serve as a relevant indicator of transferability to $C'$. Second, we compute the separability using architectures that are fully finetuned or trained from scratch on $C$. Finally, we verify whether this separability  (post-training) aligns with the separability computed with the pretrained networks.

In the context of Few-Shot learning experiments, the computations are similar, but now $C'$ corresponds to novel classes and $C$ corresponds to base classes of the datasets.

\subsection{Finetuning}
In this experiment, we use a pretrained network and identify the promising subsets of classes for finetuning using our methodology. We first decide to focus the experiments on the dataset of CIFAR10~\cite{cifar10} since it is a standard benchmark in computer vision and it is composed of $N = 10$ classes, allowing us to consider various size of subsets. Other datasets with larger number classes will make practically impossible to compute all the possible subsets. Statically we need a great number of experiments on different subsets to highlight trends which will not be possible on huge dataset where a single training takes days. 

We generate all the possible subsets consisting of $n \in \{2,4,6,8\}$ classes. We refer to CIFAR\_$n$ when the dataset is restricted to $n$ classes. As described in the methodology section, for each subset of $n$ classes we compute \textbf{separability} before finetuning and after finetuning. For subsets of sizes $n=4$ or $6$ this includes $210$ possibilites.

In Appendix, we report the correlation between the two metrics as well as the slope and intercept of the regression line for the VIT-8 and Resnet50. We note that in both cases, the correlation is high indicating that separability before finetuning is a relevant indicator of finetuning performance. 

We also conduct an analysis of the classes that have the most impact on the finetuning performance. To achieve this, we calculate what we call the class separability $sep(I)$: this involves considering all runs where a class $I$ is selected and then averaging the achieved separability of the runs (detailed computation is in the Appendix). The values of this metric are presented in Appendix. A similar Figure is also given when Training From Scratch (See Figure~\ref{fig:separabilityfromscratchvit8cifar6})
both before and after the finetuning process. Notably, we observe a high correlation between class separability before and after finetuning on CIFAR10. With the Resnet50 model, the correlation is $0.90$ on CIFAR\_4 and $0.86$ on CIFAR\_6, while with the VIT-8 model, it is $0.96$ and $0.95$ respectively.

Thus, we can confidently identify the classes that should lead to the best finetuning performance. For $n=6$, we read the classes ``ship'', ``truck'', ``bird'', ``automobile'' and ``airplane''. These classes are actually included in the subset of classes that lead to the best finetuning performance (within the 210 runs computed): (``bird'', ``cat'', ``deer'', ``dog'', ``horse'', ``ship''). On CIFAR\_4 the best subset is (``deer'', ``dog'', ``horse'', ``ship''). Animal classes seem to be the ones leading to the best generalization, although optimal subsets always contain at least one non animal class.

\begin{figure}[h]
  \centering
  \resizebox{0.4\textwidth}{!}{
\begin{tikzpicture}
\begin{axis}[
  xlabel={Number of Pairs Separated by pretrained network},
  ylabel={Number of Pairs Separated after training },
  xmin=33.5,  
  xmax=37,  
  ymin=35,  
  ymax=37.2,   
  minor x tick num=1, 
  minor y tick num=1, 
  grid=both,
  height=8cm,
  width = 8cm,
  every tick/.style={black},
  clip = false,
  clip mode=individual,
  nodes near coords=\empty, 
  visualization depends on={y \as \labelY},
  nodes near coords style={anchor=center, font=\tiny, inner sep=2pt},
  legend style={font=\small,    
  at={(0.8, 0.8)}, 
   anchor=north},
]

\addplot[only marks, mark=*, blue] coordinates {(33.94047546386719, 35.27381134033203)};
\addplot[only marks, mark=*, green] coordinates {(33.83333206176758, 35.08333206176758)};
\addplot[only marks, mark=*, red] coordinates {(35.44047546386719, 36.77381134033203)};
\addplot[only marks, mark=*, yellow] coordinates {(34.91666793823242, 36.02381134033203)};
\addplot[only marks, mark=*, magenta] coordinates {(36.14285659790039, 37.02381134033203)};
\addplot[only marks, mark=*, cyan] coordinates {(35.28571319580078, 36.57143020629883)};
\addplot[only marks, mark=*, black] coordinates {(35.42856979370117, 36.35714340209961)};
\addplot[only marks, mark=*, olive] coordinates {(35.03571319580078, 36.33333206176758)};
\addplot[only marks, mark=*, teal] coordinates {(34.154762268066406, 35.488094329833984)};
\addplot[only marks, mark=*, purple] coordinates {(34.154762268066406, 35.4523811340332)};

\legend{airplane, automobile, bird, cat, deer, dog, frog, horse, ship, truck}

\end{axis}
\end{tikzpicture}}
  \caption{Comparison of the number of pairs separated by each class. The x-axis represents the results from the Pretrained VIT-8, while the y-axis corresponds to the Resnet18 trained from scratch on 6 classes of CIFAR10. The analysis highlights the most promising classes, such as deer, frog, horse, and cat. Interestingly, some of these promising classes are also part of the best subset when training from scratch (automobile, bird, cat, deer, horse, truck).}
  \label{fig:separabilityfromscratchvit8cifar6}
\end{figure}
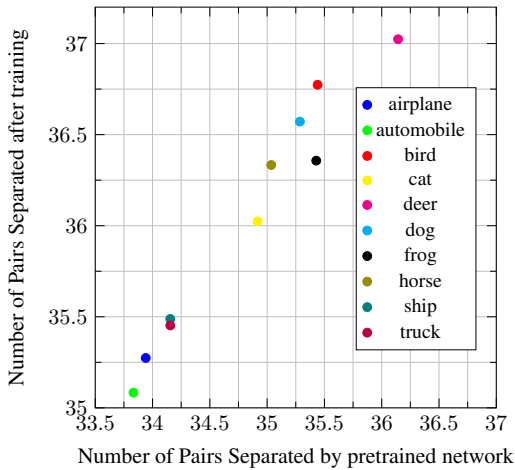
It is worth pointing out that there are slight differences between the graphs given by the VIT-8 and the Resnet50 (c.f. Appendix). Differences could be explained by the training procedure: VIT-8 is a recent architecture trained using a contrastive loss function and benefits from various data augmentation techniques and recent regularizations.

Additionally, we compute the same metrics by considering pairs instead of classes. Once again, we observe an high correlation on CIFAR\_6 and CIFAR\_4. Detailed correlations are provided in the Appendix. Interestingly, in both settings, the pair that achieved the worst separability, i.e. the pair that separates the least pairs of classes, is dog vs. cat. This observation may suggest that the learned hyperplane is overly specific to this particular pair. The best pair of classes turns out to be airplane/bird.

\subsection{Training From Scratch}
In these experiments we investigate whether we could identify the subset leading to the best generalization (in the sense of transfer) when we train a network from scratch.

We use a Resnet18~\cite{resnet} which is a standard and competitive architecture in Computer Vision~\cite{touvronresnet}. On CIFAR10 we trained on subsets of $n=4$ and $n=6$ classes. Again we use either the pretrained Resnet50 or VIT-8 to compute a representation of the data in order to compute the separability of the subsets before training. Contrary to finetuning, it is uncertain that the separability given by the pretrained network will be a good indicator of the separability after training since the network are not trained on the same dataset and they differ in terms of architecture and training procedure.

However, as reported in Table~\ref{tab:coefffromscratch}, the pretrained networks still serve as relevant indicators of performance, with high correlation coefficients, especially for the class separability. Similarly to what we did in finetuning, we plot the mean separability of each class on Figure~\ref{fig:separabilityfromscratchvit8cifar6}, providing insights into identifying the most promising classes. Interestingly, compared to finetuning, the optimal subsets showed small differences. For example, for $n=6$, the optimal subset is (automobile/bird/cat/deer/horse/truck). As depicted on Figure~\ref{fig:separabilityfromscratchvit8cifar6} most of these classes are considered as considered promising. We note few differences with finetuning.

We carry out a similar experiment on FASHION-MNIST~\cite{fashion} also composed of 10 classes. The class ``Dress'' appears as the most promising (see Appendix), which makes sense as this class differs from the other clothing classes. Remarkably on both dataset, none of the models learned from this optimal subset of classes are equivalent.

\subsection{Few-Shot Learning}
In this study, we conduct experiments in the domain of few-shot learning, which serves as a typical scenario to evaluate a model's ability to generalize to new classes. Specifically, we focus on inductive few-shot learning, where the predictions are solely based on the few labeled examples. To assess the performance, we employ an nearest mean classifier (NCM) approach, which is commonly employed in the context of inductive few-shot learning~\cite{nncfewshot}. 

In line with our approach in other sections, we assess the separability of novel classes using representations obtained from a pretrained network, specifically Resnet-50. However, in this case, the separability computation focuses solely on models inferred from pairs of base classes, as few-shot learning exclusively trains on these classes. For each dataset, we identify potential best and worst sets of pairs.

Once more, it is uncertain whether the separability provided by the pretrained network serves as a reliable indicator for few-shot learning. During evaluation, we consider other networks that are trained solely on the base classes, and few-shot evaluation differs significantly from standard classification. 
To conduct our experiments, we utilize the networks provided by~\cite{easy}. We conduct our experiments on standard few-shot learning datasets, including MINIIMAGENET~\cite{miniimagenet}, TIEREDIMAGENET~\cite{tieredimagenet}, FC100~\cite{fc100}, and CIFAR-FS~\cite{cifarfs}, which are widely recognized benchmarks in the field. To assess the performance, we utilize the Nearest Class Mean (NCM) approach on both the best and worst sets of class pairs under the 1-shot learning scenario, where each novel class has only a single labeled example. Additionally, we establish a baseline by reporting the performance on all class pairs. To ensure rigorous results, we conduct 10000 runs for each case and compute the average accuracy and confidence interval (95\%) across these runs.
\\

As shown in Table~\ref{tab:fewshot}, the worst sets consistently yield lower performance compared to both the baseline and the best sets, across most datasets except for TIEREDIMAGENET. Notably, the best sets consistently outperform other sets across all datasets. Specifically, on MINIIMAGENET, the most challenging pairs include ``crab" vs ``ants," ``lion" vs ``bus," and ``Siberian Husky" vs ``Dalmatian". We provide visual examples of these challenging pairs in Appendix. Upon examination of the visual examples, it becomes apparent that ``crab" and ``ants" share similar colors, textures, and structural patterns. Similarly, the classes ``lion" and ``bus" are represented with a similar background of grass. Additionally, the classes ``Dalmatian" and ``Husky" are both dog breeds, with similar color patterns and texture, making them harder to distinguish.

 \begin{table}[t]
  \caption{Accuracy averaged over 10000 runs in Inductive Few Shot Learning (2-Ways, 1-Shot 15-Queries).}
  \label{tab:my_label}
  \centering
  \begin{tabular}{llll}
  \toprule
   Dataset & Best & Worst & Base \\
  \midrule
  MINI-I&$ 0.94 \pm 0.12$ & $0.86 \pm 0.12$ &$ 0.87 \pm 0.11 $\\
  CIFARFS & $ 0.87 \pm 0.11 $ & $0.83 \pm 0.15$& $0.85 \pm 0.14$ \\
  FC100 & $0.74 \pm 0.10$ & $0.68 \pm 0.14$ & $0.70 \pm 0.15 $\\
  TIERED-I & $0.95 \pm 0.07$ & $0.94 \pm 0.07$ & $0.95 \pm 0.07$\\
 \bottomrule
  \end{tabular}
  \label{tab:fewshot}
 \end{table}

\section{Limitations}
Our proposed framework comes with limitations. The framework focuses on models based on pairs, and in practice, we potentially overlook interactions that may emerge when subsets encompass multiple pairs. However, it's worth noting that our theoretical framework could be easily extended to include subsets of several pairs, instead of singleton of pair, offering a potential solution to overcome this limitation. Furthermore, our assertion about identifying promising subsets during training from scratch might be overly ambitious and constrained, given that the effectiveness of subset identification  relies on pretrained networks.

\section{Further Resarch}
Future research involves using our framework to explore various aspects of transfer learning. For example, the impact of training procedures impact on inferred representations and transferability. By systematically varying settings like data augmentation, loss functions, and optimization algorithms, researchers can gain insights into how these factors affect pretrained models' generalization capabilities, potentially leading to more robust models.

Additionally, our formalism opens up possibilities for constructing challenging few-shot learning datasets. By carefully designing subsets of classes with specific characteristics, such as similarity, diversity, or hierarchical relationships, researchers can create benchmarks that thoroughly assess the generalization capabilities of few-shot learning algorithms.

\section{Conclusion}
This paper addressed the concept of transfer learning in classification tasks by introducing a novel theoretical framework. We define an order relationship of learning, allowing us to precisely define transferability between different sets of classes. By conducting experiments on multiple dataset, we demonstrated that the separability achieved by classes before fine-tuning is a strong indicator of the transferability potential during the fine-tuning process. This finding opens up new possibilities for predicting model performance and optimizing transfer learning practices.

Additionally, we explored training from scratch and few-shot learning scenarios, finding that the separability provided by pretrained networks remains relevant for assessing future performance. Overall, this work contributes to a deeper understanding of transfer learning mechanisms and paves the way for enhanced model generalization and transferability in various real-world applications.

The proposed framework opens up new research possibilities, enabling the study of the impact of class sampling, architectures, and training procedures on representations and transferability. 

{
    \small
    \bibliographystyle{ieeenat_fullname}
    \bibliography{main}
}

\input{sec/X_suppl}
\end{document}

%% file: preamble.tex
\usepackage[dvipsnames]{xcolor}


%% file: sec/X_suppl.tex
\clearpage
\setcounter{page}{1}
\maketitlesupplementary

\section{Proofs}
\begin{proof} of Theorem 1.
 Let us consider such a set of fundamental pairs $P$ and a set of separators $S$ of cardinal $\mathcal{F}(C)$.\\

 Let us consider $p$ an element of $P$. By construction of $P$, $p$ is separated by a single separator of $S$. Otherwise the element would have an equivalent leading to a contradiction. Hence the cardinal of $P$ is smaller than or equal to the cardinal of $S$.\\
 
 Conversely, let us consider a separator $M \in S$. We show it separates one single fundamental pair. By definition of $P$, the following set: $P \cap \mathcal{C}(M)$ has no more than one element. By contradiction, let us assume that this set is empty, i.e, $\mathcal{C}(M)$ contains only pairs which are not fundamental. So, for any element $c \in \mathcal{C}(M)$, there is a fundamental pair $p \in P$ s.t $c < p$. But $p$ is necessarily separated by a separator $M'\in S$ and consequently $c$ is also separated by $M'$. At the end, we have two separators $M,M'$ for $c$. Since the reasoning is true for any $c \in \mathcal{C}(M)$, then $S\setminus M$ separates all pairs of $C$. Hence a contradiction with the cardinal of $S$ being equal to $\mathcal{F}(C)$. Therefore, each model of $S$ separates a single fundamental pair; the cardinal of $S$ is smaller or equal to the cardinal of $P$. In conclusion, the cardinal of $P$ equals $\mathcal{F}(C)$.
\end{proof} 
\begin{proof}. We provide on example reaching the upper bound on $\mathcal{F}(C)$.
Let be $n$ an even number, we define $C = \{\{1,2 \},\{3,4\}, \cdots, \{n-1,n \} \}$ composed of $n$ sets. Let us consider a pair $I,J \subset C$. Without loss of generality, we assume that $I<J$. We define the model $M_{I,J}$ a biparition $(M^1_{I,J},M^2_{I,J})$ as follows:
\begin{align*}
 I &\subset M^1_{I,J} \text{ and } J\subset M^2_{I,J},\\
 \forall X & \in \{1, \cdots, n\} \setminus(I\cup J), \ X \mod 2 = 0 \Rightarrow X \subset M^1_{I,J},\\
 \forall X & \in \{1, \cdots, n\}\setminus(I\cup J), \ X \mod 2 = 1 \Rightarrow X \subset M^2_{I,J}.\\
\end{align*}
 We have constructed $\binom{n}{2}$ models and by definition none of the models are equivalent.
\end{proof}
 \begin{proof} We provide one example reaching the higher bound. Let be $n$ an even number, we define $C = \{\{1,2 \},\{3,4\}, \cdots, \{n-1,n \} \}$ composed of $n$ sets. Let us consider a pair $I,J \subset C$. Without loss of generality, we assume that $I<J$. We define the model $M_{I,J}$ a biparition $(M^1_{I,J},M^2_{I,J})$ as follows:
\begin{align*}
 I &\subset M^1_{I,J} \text{ and } J\subset M^2_{I,J},\\
 \forall X & \in \{1, \cdots, n\} \setminus(I\cup J), \ X \mod 2 = 0 \Rightarrow X \subset M^1_{I,J},\\
 \forall X & \in \{1, \cdots, n\}\setminus(I\cup J), \ X \mod 2 = 1 \Rightarrow X \subset M^2_{I,J}.\\
\end{align*}
 We have constructed $\binom{n}{2}$ models and by definition none of the models are equivalent.
\end{proof}

\begin{proof} We provide one example that reaches the lower bound.
 Let us consider $C = \{0,1\}^k$ the set of natural integers between $0$ and $2^{k-1}$ written in base $2$ with $k$ digits. For instance, when $k = 2$, $C = \{00,01,10,11\}$. By definition of $C$, $n= 2^k$.\\

 Let us define $k$ models ($M_0, \cdots, M_{k-1})$ such that for $i\in \{ 0, \cdots, k-1\}$ the partition $M^1_k$ contains the numbers with a $k\textsuperscript{th}$ digit equal to $0$ and $M^2_k$ the number with a $k\textsuperscript{th}$ digit equal to $1$.\\
 Let us consider a pair $I,J \in C$, by construction of $C$, $I$ and $J$ differ from one bit at least at position $k$. Thus $I \perp_{M_{k}} J$ and we set $M_{I,J} = M_k$.\\
 Hence $\mathcal{F}(C) = k = log_2(n)$
 \end{proof}
\section{Further Research}
Further research can build upon the presented framework to investigate several intriguing questions and explore various aspects of transfer learning. For instance, the framework can be leveraged to examine the impact of different training procedures on the inferred representations and transferability. By systematically varying the training settings, such as data augmentation techniques, loss functions, and optimization algorithms, researchers can gain deeper insights into how these factors affect the generalization capabilities of pretrained models. Understanding the nuances of training procedures could lead to the development of more robust and transferable models.

Additionally, the formalism introduced in this work opens up possibilities for constructing challenging few-shot learning datasets. By carefully designing subsets of classes with specific characteristics, such as similarity, diversity, or hierarchical relationships, researchers can create benchmarks that thoroughly assess the generalization capabilities of few-shot learning algorithms. These datasets can aid in evaluating the effectiveness of different few-shot learning methods and shed light on the strategies that work best.

Moreover, the framework's potential extends to investigating the transferability of representations across different domains and tasks. Researchers can explore how pretrained models generalize across different datasets, domains. 
\section{Proposed metrics}
Let us explain in details the different metrics that we introduce in the experimental part.

\section{Class separability:}
Let us consider a class $I$, and we will now define and compute its separability as outlined below. We take all the runs, where the class $I$ appears in the selected subset of classes. For each run, we utilize the corresponding model and associated heads—these are all the heads $h$ assigned to pairs that include class $I$. We then consider the smallest error given by the heads for each pair and check if it is smaller than $\varepsilon$.

For instance, let ${K,L}$ be a pair we are considering. The separability $\text{sep}_{K,L}(I)$ is determined as follows:\[ \text{sep}_{K,L}(I) =
\begin{cases} 
0 & \text{if } \min_{j\in \{1,\cdots,n\}} \hat{\mathbf{P}}_{err}(h_{I,K}, K, L) > \varepsilon \\
1 & \text{otherwise.}
\end{cases}
\] 
Finally, we sum the separabilites of all pairs so that: $ \text{sep}(I) = \sum_{I=1}^N \sum_{J=I+1}^N \text{sep}_{K,L}(I)$.
\section{Pair separability:}
We conduct a computation analogous to the one performed for class separability, with a focus on pairs. Let us consider a pair $P$, we proceed as follows. We consider all the runs where the pair $P$ is included in the selected pairs of classes. For each run, we take the corresponding model and the associated head $h$ assigned to the pair $P$. We then count the pairs that can be separated to obtain $\text{sep}(P)$.

\section{Code Appendix}
Below, we provide an overview of the code used for conducting the experiments presented in the paper. Code is available in an archive. Dataset are downloadable with pytorch, except for CUB200 which available here~\url{https://www.vision.caltech.edu/datasets/cub_200_2011/}, MINIIMAGNET and TIEREDIMAGENET available on Kaggle.

\subsection*{Predicting Subset Separability}
\label{subsec:predict}

In the \texttt{predict\_potential\_subsets} folder, you will find the code related to predicting the separability of subsets of classes:

\begin{enumerate}
    \item \texttt{computes\_regressor.py}: This script train linear a regressor (one head per pair) which is later used to assess the separability of subsets of classes.
    \item \texttt{generates\_features.py}: This script generates the features required for computing the separability.
    \item \texttt{computes\_dict\_couples\_head.py}: This script computes a matrix of accuracies $\mathbf{M}$, where $\mathbf{M}[I,J]$ gives the accuracy of the head learned on pair $I$ when separating pair $J$.
\end{enumerate}

\subsection*{Training from Scratch and Few-shot Learning}
\label{subsec:train}

In the \texttt{training\_from\_scratch\_and\_fewshot} folder, we provide the \texttt{accuracy\_scores} matrix that has already been computed. To run the experiments use this scrip \texttt{it\_training\_classes.py}, to train on random subsets of classes. You can manually adjust the dataset and the number of classes.

\subsection*{Few-shot Learning}
\label{subsec:fewshot}

In the \texttt{fewshot} folder, we have included scripts for few-shot learning experiments:

\begin{enumerate}
    \item \texttt{fw\_computes\_best\_worst\_sets.py}: This script computes the best and worst subsets for fewshot.
    \item \texttt{fw\_test.py}: Run this script to launch the few-shot test (NCM). Features are required.
\end{enumerate}

\section{Additional Experiment}
We conduct an additional experiment involving fine-tuning. Specifically, we focus on the fine-grained CUB200 dataset, which consists of 200 distinct bird classes (accessible at: \url{https://www.vision.caltech.edu/datasets/cub_200_2011/}). Our approach includes finetuning the networks on subsets of 60 classes, and we generate 300 distinct subsets.

Furthermore, we present the outcomes of a detailed analysis in Table~\ref{tab:corrfinetuningCUB}, which examines the correlation and slope coefficients of the regression analysis between the separability prior to and following the fine-tuning process for both the VIT-8 and Resnet50 networks on the CUB200 dataset. Notably, we utilized $\varepsilon = 0.25$ to measure separability.

\section{Additional Tables and Figures}
 \begin{table}[H]
\caption{Correlation and Slope Coefficient of the regression between the separability before and after fine-tuning for VIT-8 and Resnet50 on the CUB200 dataset. We use \(\varepsilon = 0.025\) to compute the separability.}

  \label{tab:my_label}
  \centering
  \begin{tabular}{llll}
  \toprule
   Network & Pearson Coeff. & Slope & Intercept \\
  \midrule
  VIT-8 & $ 0.88 $ & $0.60$ & $8090$ \\
  Resnet50 & $0.71$ & $0.63$& $7766$ \\
 \bottomrule
  \end{tabular}
  \label{tab:corrfinetuningCUB}
 \end{table}
 
 \begin{table}[H]
  \caption{Correlation and Slope Coefficient of the regression between the separability before and after finetuning for the pretrained Resnet50.}
  \label{tab:my_label}
  \centering
  \begin{tabular}{llll}
  \toprule
   Dataset & Pearson Coeff. & Slope & Intercept \\
  \midrule
  CIFAR10\_4 & $ 0.76 $ & $0.73$ & $0.89$ \\
  CIFAR10\_6 & $0.79$ & $0.89$& $2.8$ \\
 \bottomrule
  \end{tabular}
  \label{tab:corrfinetuningResnet}
 \end{table}

\begin{table*}[h]
  \caption{Correlation between multiple metrics given by a pretrained network and separability when training from scratch for the VIT-8.}
  \label{tab:my_label}
  \centering
  \begin{tabular}{lllll}
  \toprule
  Metric & CIFAR10\_4 & CIFAR10\_6 & FASHION & MNIST\\
  \midrule
  Separability & $ 0.78 $ & $0.44$ & $0.88$ & $0.58$ \\
  Class Separability & $0.94$ & $0.97$& $0.96$ &$0.73$\\
  Pair separability & $0.92$ & $0.95$ & $0.94$ &$0.68$\\
 \bottomrule
  \end{tabular}
  \label{tab:coefffromscratchvit8}
 \end{table*}
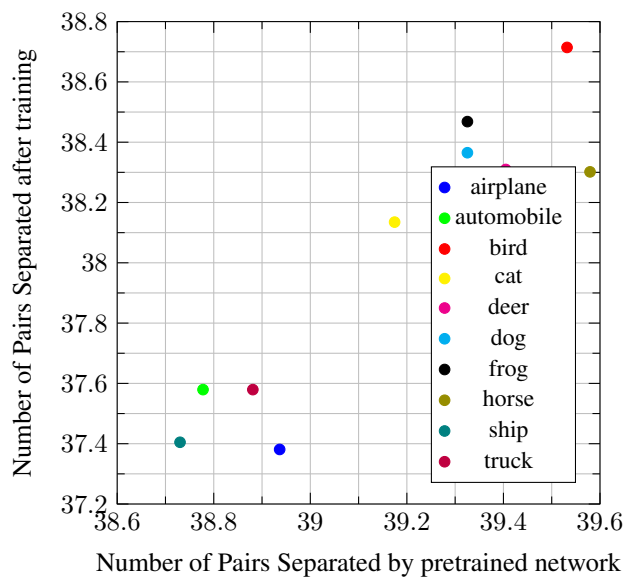
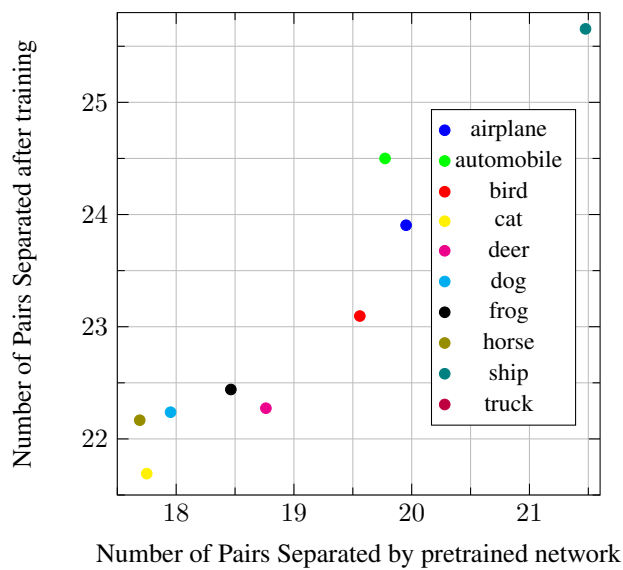
\begin{figure*}[h]
\centering
\begin{subfigure}{\columnwidth}
  \centering
\begin{tikzpicture}
\begin{axis}[
  xlabel={Number of Pairs Separated by pretrained network},
  ylabel={Number of Pairs Separated after training },
  xmin=38.6,  
  xmax=39.6,  
  ymin=37.2,  
  ymax=38.8,   
  minor x tick num=1, 
  minor y tick num=1, 
  grid=both,
  height=8cm,
  width = 8cm,
  every tick/.style={black},
  clip = false,
  clip mode=individual,
  nodes near coords=\empty, 
  visualization depends on={y \as \labelY},
  nodes near coords style={anchor=center, font=\tiny, inner sep=2pt},
  legend style={font=\small,    
  at={(0.8, 0.7)}, 
   anchor=north},
]

\addplot[only marks, mark=*, blue] coordinates {(38.93650817871094, 37.380950927734375)};
\addplot[only marks, mark=*, green] coordinates {(38.77777862548828, 37.57936477661133)};
\addplot[only marks, mark=*, red] coordinates {(39.53174591064453, 38.71428680419922)};
\addplot[only marks, mark=*, yellow] coordinates {(39.17460250854492, 38.13492202758789)};
\addplot[only marks, mark=*, magenta] coordinates {(39.404762268066406, 38.30952453613281)};
\addplot[only marks, mark=*, cyan] coordinates {(39.32539749145508, 38.36507797241211)};
\addplot[only marks, mark=*, black] coordinates {(39.32539749145508, 38.46825408935547)};
\addplot[only marks, mark=*, olive] coordinates {(39.57936477661133, 38.30158615112305)};
\addplot[only marks, mark=*, teal] coordinates {(38.730159759521484, 37.404762268066406)};
\addplot[only marks, mark=*, purple] coordinates {(38.880950927734375, 37.57936477661133)};

\legend{airplane, automobile, bird, cat, deer, dog, frog, horse, ship, truck}

\end{axis}
\end{tikzpicture}
  \caption{Comparison of the number of pairs separated by each class. The x-axis represents the results from the Pretrained Resnet-50, while the y-axis corresponds to the same network finetuned on 6 classes of CIFAR10. The analysis highlights the most promising classes, such as bird, frog, dog and horse.}

\end{subfigure}%
\hfill
\begin{subfigure}{\columnwidth}
  \centering
\begin{tikzpicture}
\begin{axis}[
  xlabel={Number of Pairs Separated by pretrained network},
  ylabel={Number of Pairs Separated after training },
  xmin=17.5,  
  xmax=21.6,  
  ymin=21.5,  
  ymax=25.8,   
  minor x tick num=1, 
  minor y tick num=1, 
  grid=both,
  height=8cm,
  width = 8cm,
  every tick/.style={black},
  clip = false,
  clip mode=individual,
  nodes near coords=\empty, 
  visualization depends on={y \as \labelY},
  nodes near coords style={anchor=center, font=\tiny, inner sep=2pt},
  legend style={font=\small,    
  at={(0.8, 0.8)}, 
   anchor=north},
]

\addplot[only marks, mark=*, blue] coordinates {(19.952381134033203, 23.904762268066406)};
\addplot[only marks, mark=*, green] coordinates {(19.7738094329834, 24.5)};
\addplot[only marks, mark=*, red] coordinates {(19.559524536132812, 23.095237731933594)};
\addplot[only marks, mark=*, yellow] coordinates {(17.75, 21.690475463867188)};
\addplot[only marks, mark=*, magenta] coordinates {(18.761905670166016, 22.2738094329834)};
\addplot[only marks, mark=*, cyan] coordinates {(17.952381134033203, 22.238094329833984)};
\addplot[only marks, mark=*, black] coordinates {(18.464284896850586, 22.440475463867188)};
\addplot[only marks, mark=*, olive] coordinates {(17.690475463867188, 22.16666603088379)};
\addplot[only marks, mark=*, teal] coordinates {(21.4761905670166, 25.654762268066406)};
\addplot[only marks, mark=*, purple] coordinates {(21.14285659790039, 24.75)};

\legend{airplane, automobile, bird, cat, deer, dog, frog, horse, ship, truck}

\end{axis}
\end{tikzpicture}
  \caption{Comparison of the number of pairs separated by each class. The x-axis represents the results from the Pretrained Resnet-50, while the y-axis corresponds to the same network finetuned on 4 classes of CIFAR10. The analysis highlights the most promising classes, such as ship, automobile, airplane, bird.}

\end{subfigure}
  \caption{Separability given by a Resnet-50 before and after finetuning on subsets of CIFAR10.}
  \label{fig:sepclassesfinetuningRes}
\end{figure*}

\begin{figure*}[h]
\centering
\begin{subfigure}{\columnwidth}
  \centering
\begin{tikzpicture}
\begin{axis}[
  xlabel={Number of Pairs Separated by pretrained network},
  ylabel={Number of Pairs Separated after training },
  xmin=41,  
  xmax=42,  
  ymin=40.8,  
  ymax=41.8,   
  minor x tick num=1, 
  minor y tick num=1, 
  grid=both,
  height=8cm,
  width = 8cm,
  every tick/.style={black},
  clip = false,
  clip mode=individual,
  nodes near coords=\empty, 
  visualization depends on={y \as \labelY},
  nodes near coords style={anchor=center, font=\tiny, inner sep=2pt},
  legend style={font=\small,    
  at={(0.8, 0.8)}, 
   anchor=north},
]

\addplot[only marks, mark=*, blue] coordinates {(41.02381134033203, 40.873016357421875)};
\addplot[only marks, mark=*, green] coordinates {(41.17460250854492, 40.97618865966797)};
\addplot[only marks, mark=*, red] coordinates {(41.738094329833984, 41.4603157043457)};
\addplot[only marks, mark=*, yellow] coordinates {(41.873016357421875, 41.595237731933594)};
\addplot[only marks, mark=*, magenta] coordinates {(41.85714340209961, 41.72222137451172)};
\addplot[only marks, mark=*, cyan] coordinates {(41.81745910644531, 41.66666793823242)};
\addplot[only marks, mark=*, black] coordinates {(41.841270446777344, 41.64285659790039)};
\addplot[only marks, mark=*, olive] coordinates {(41.33333206176758, 41.16666793823242)};
\addplot[only marks, mark=*, teal] coordinates {(41.30158615112305, 41.03174591064453)};
\addplot[only marks, mark=*, purple] coordinates {(41.18254089355469, 40.9603157043457)};

\legend{airplane, automobile, bird, cat, deer, dog, frog, horse, ship, truck}

\end{axis}
\end{tikzpicture}
  \caption{Comparison of the number of pairs separated by each class. The x-axis represents the results from the Pretrained VIT-8, while the y-axis corresponds to the same network finetuned on 6 classes of CIFAR10. The analysis highlights the most promising classes, such as truck, dog, frog.}

\end{subfigure}%
\hfill
\begin{subfigure}{\columnwidth}
  \centering
\begin{tikzpicture}
\begin{axis}[
  xlabel={Number of Pairs Separated by pretrained network},
  ylabel={Number of Pairs Separated after training },
  xmin=25,  
  xmax=28,  
  ymin=26,  
  ymax=28.5,   
  minor x tick num=1, 
  minor y tick num=1, 
  grid=both,
  height=8cm,
  width = 8cm,
  every tick/.style={black},
  clip = false,
  clip mode=individual,
  nodes near coords=\empty, 
  visualization depends on={y \as \labelY},
  nodes near coords style={anchor=center, font=\tiny, inner sep=2pt},
  legend style={font=\small,    
  at={(0.8, 0.8)}, 
   anchor=north},
]

\addplot[only marks, mark=*, blue] coordinates {(27.64285659790039, 28.440475463867188)};
\addplot[only marks, mark=*, green] coordinates {(26.154762268066406, 26.845237731933594)};
\addplot[only marks, mark=*, red] coordinates {(25.428571701049805, 26.404762268066406)};
\addplot[only marks, mark=*, yellow] coordinates {(25.202381134033203, 26.119047164916992)};
\addplot[only marks, mark=*, magenta] coordinates {(25.297618865966797, 26.35714340209961)};
\addplot[only marks, mark=*, cyan] coordinates {(25.047618865966797, 26.08333396911621)};
\addplot[only marks, mark=*, black] coordinates {(25.261905670166016, 26.119047164916992)};
\addplot[only marks, mark=*, olive] coordinates {(25.25, 26.154762268066406)};
\addplot[only marks, mark=*, teal] coordinates {(28.0, 28.464284896850586)};
\addplot[only marks, mark=*, purple] coordinates {(26.0, 26.7261905670166)};
;

\legend{airplane, automobile, bird, cat, deer, dog, frog, horse, ship, truck}

\end{axis}
\end{tikzpicture}
  \caption{Comparison of the number of pairs separated by each class. The x-axis represents the results from the Pretrained VIT-8, while the y-axis corresponds to the same network finetuned on 4 classes of CIFAR10. The analysis highlights the most promising classes, such as ship, airplane, automobile and truck.}
 
\end{subfigure}
  \caption{Separability given by a VIT-8 before and after finetuning on subsets of CIFAR10.}
\end{figure*}
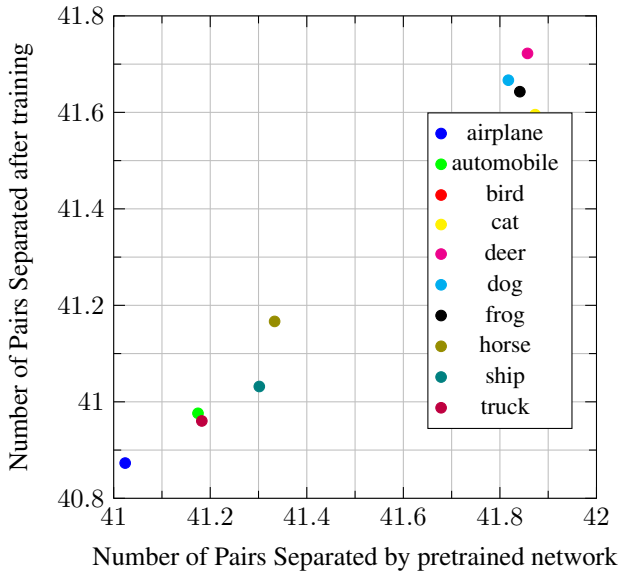
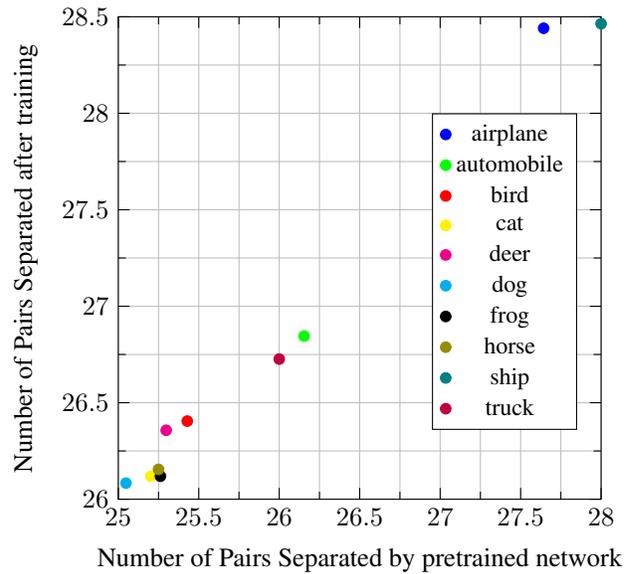

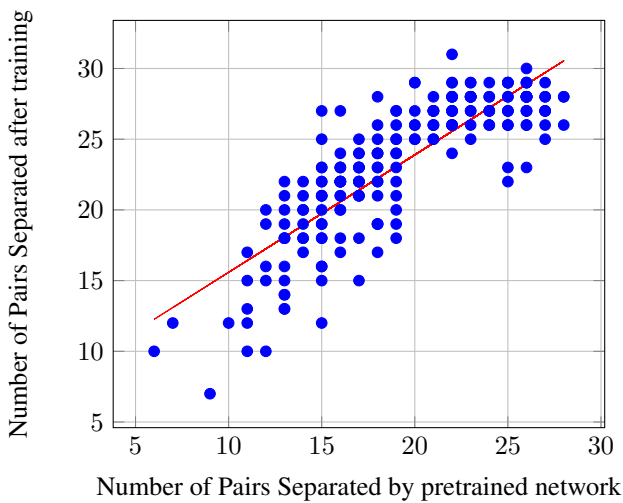
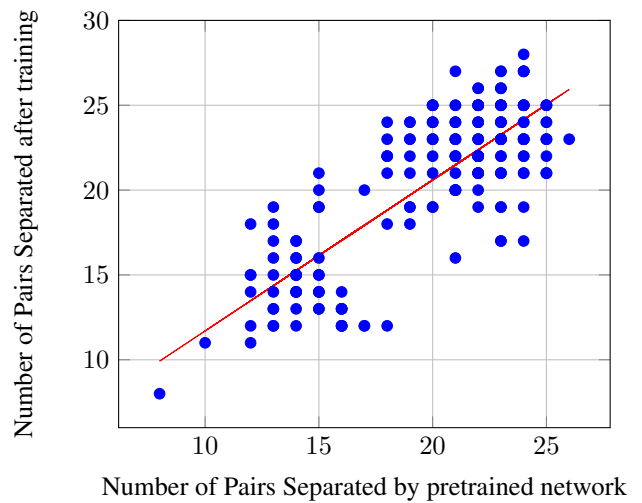
\begin{figure*}[h]
\centering
 \begin{subfigure}{0.45\textwidth}
  \centering
  \begin{tikzpicture}
   \begin{axis}[ xlabel={Number of Pairs Separated by pretrained network},
  ylabel={Number of Pairs Separated after training },    grid=both,height=7cm,]
    \addplot[blue,only marks,mark=*,    ] table [x=x, y=y, col sep=comma]    {CSV/separabilityfinetuning_cifar10_resnet50_4_classes.csv};
    \addplot[smooth,red,] table [x=x, y={create col/linear regression={y=y}}, col sep=comma] {CSV/separabilityfinetuning_cifar10_resnet50_4_classes.csv};
   \end{axis}
  \end{tikzpicture}
  \caption{Subsets of 4 classes of CIFAR10.}
  \label{fig:fig1}
 \end{subfigure}%
 \hfill
 \begin{subfigure}{0.45\textwidth}
  \centering
  \begin{tikzpicture}
   \begin{axis}[ xlabel={Number of Pairs Separated by pretrained network},
  ylabel={Number of Pairs Separated after training },    grid=both,height=7cm,]
    \addplot[blue,only marks,mark=*,    ] table [x=x, y=y, col sep=comma]    {CSV/separabilityfinetuning_cifar10_resnet50_6_classes.csv};
    \addplot[smooth,red,] table [x=x, y={create col/linear regression={y=y}}, col sep=comma] {CSV/separabilityfinetuning_cifar10_resnet50_6_classes.csv};
   \end{axis}
  \end{tikzpicture}
  \caption{Subsets of 6 classes of CIFAR10.}
  \label{fig:fig2}
 \end{subfigure}
 \caption{Resnet50: Number of Pairs Separated before vs after finetuning on subsets of CIFAR10.}
 \label{fig:overall}
\end{figure*}
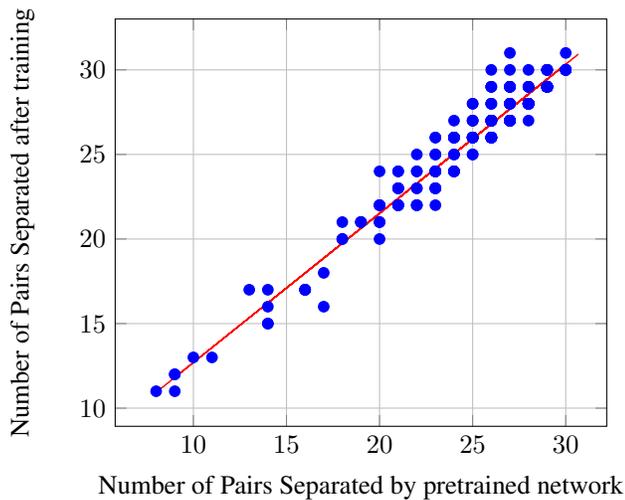
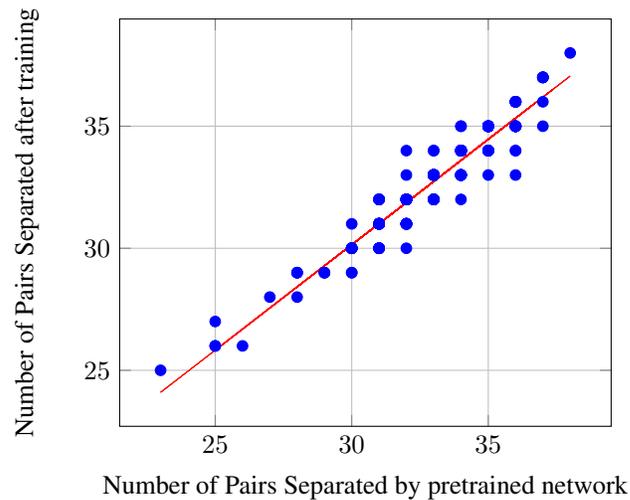
\begin{figure*}[h]
\centering
 \begin{subfigure}{0.45\textwidth}
  \centering
  \begin{tikzpicture}
   \begin{axis}[ xlabel={Number of Pairs Separated by pretrained network},
  ylabel={Number of Pairs Separated after training },    grid=both,height=7cm,]
    \addplot[blue,only marks,mark=*,    ] table [x=x, y=y, col sep=comma]    {CSV/separabilityfinetuning_cifar10_vit8_4_classes.csv};
    \addplot[smooth,red,] table [x=x, y={create col/linear regression={y=y}}, col sep=comma] {CSV/separabilityfinetuning_cifar10_vit8_4_classes.csv};
   \end{axis}
  \end{tikzpicture}
  \caption{Subsets of 4 classes of CIFAR10.}
  \label{fig:fig1}
 \end{subfigure}%
 \hfill
 \begin{subfigure}{0.45\textwidth}
  \centering
  \begin{tikzpicture}
   \begin{axis}[ xlabel={Number of Pairs Separated by pretrained network},
  ylabel={Number of Pairs Separated after training },    grid=both,height=7cm,]
    \addplot[blue,only marks,mark=*,    ] table [x=x, y=y, col sep=comma]    {CSV/separabilityfinetuning_cifar10_vit8_6_classes.csv};
    \addplot[smooth,red,] table [x=x, y={create col/linear regression={y=y}}, col sep=comma] {CSV/separabilityfinetuning_cifar10_vit8_6_classes.csv};
   \end{axis}
  \end{tikzpicture}
  \caption{Subsets of 6 classes of CIFAR10.}
  \label{fig:fig2}
 \end{subfigure}
 \caption{VIT8: Number of Pairs Separated before vs after finetuning on subsets of CIFAR10.}
 \label{fig:overall}
\end{figure*}

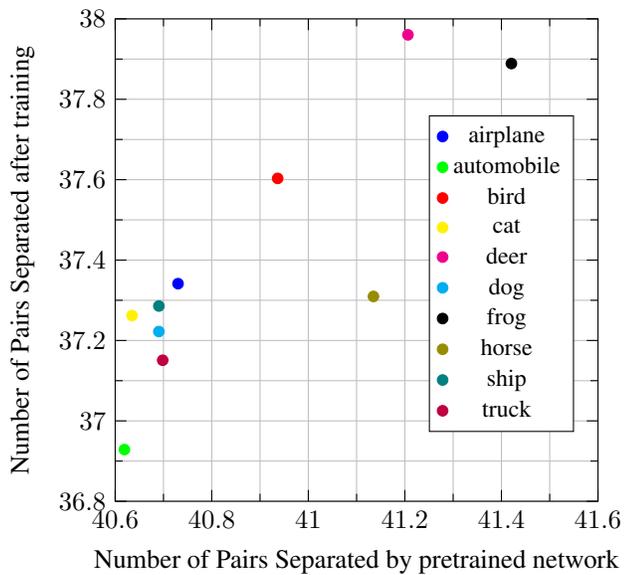
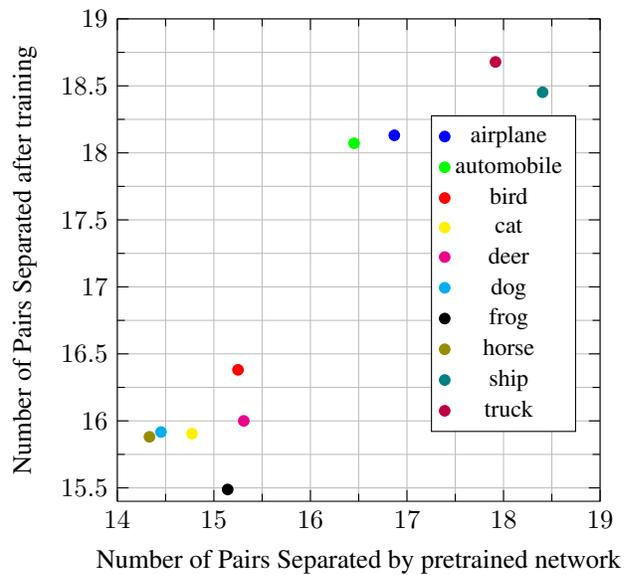
\begin{figure*}[h]
\centering
\begin{subfigure}{\columnwidth}
  \centering
\begin{tikzpicture}
\begin{axis}[
  xlabel={Number of Pairs Separated by pretrained network},
  ylabel={Number of Pairs Separated after training },
  xmin=40.6,  
  xmax=41.6,  
  ymin=36.8,  
  ymax=38,   
  minor x tick num=1, 
  minor y tick num=1, 
  grid=both,
  height=8cm,
  width = 8cm,
  every tick/.style={black},
  clip = false,
  clip mode=individual,
  nodes near coords=\empty, 
  visualization depends on={y \as \labelY},
  nodes near coords style={anchor=center, font=\tiny, inner sep=2pt},
  legend style={font=\small,    
  at={(0.8, 0.8)}, 
   anchor=north},
]

\addplot[only marks, mark=*, blue] coordinates {(40.730159759521484, 37.341270446777344)};
\addplot[only marks, mark=*, green] coordinates {(40.619049072265625, 36.92856979370117)};
\addplot[only marks, mark=*, red] coordinates {(40.93650817871094, 37.60317611694336)};
\addplot[only marks, mark=*, yellow] coordinates {(40.63492202758789, 37.261905670166016)};
\addplot[only marks, mark=*, magenta] coordinates {(41.20634841918945, 37.9603157043457)};
\addplot[only marks, mark=*, cyan] coordinates {(40.69047546386719, 37.22222137451172)};
\addplot[only marks, mark=*, black] coordinates {(41.42063522338867, 37.88888931274414)};
\addplot[only marks, mark=*, olive] coordinates {(41.13492202758789, 37.30952453613281)};
\addplot[only marks, mark=*, teal] coordinates {(40.69047546386719, 37.28571319580078)};
\addplot[only marks, mark=*, purple] coordinates {(40.69841384887695, 37.150794982910156)};

\legend{airplane, automobile, bird, cat, deer, dog, frog, horse, ship, truck}

\end{axis}
\end{tikzpicture}
  \caption{Comparison of the number of pairs separated by each class. The x-axis represents the results from the Pretrained Resnet-50, while the y-axis corresponds to the Resnet18 trained from scratch on 6 classes of CIFAR10. The analysis highlights the most promising classes, such as deer, frog, horse, and bird.}

\end{subfigure}%
\hfill
\begin{subfigure}{\columnwidth}
  \centering
\begin{tikzpicture}
\begin{axis}[
  xlabel={Number of Pairs Separated by pretrained network},
  ylabel={Number of Pairs Separated after training },
  xmin=14,  
  xmax=19,  
  ymin=15.4,  
  ymax=19,   
  minor x tick num=1, 
  minor y tick num=1, 
  grid=both,
  height=8cm,
  width = 8cm,
  every tick/.style={black},
  clip = false,
  clip mode=individual,
  nodes near coords=\empty, 
  visualization depends on={y \as \labelY},
  nodes near coords style={anchor=center, font=\tiny, inner sep=2pt},
  legend style={font=\small,    
  at={(0.8, 0.8)}, 
   anchor=north},
]

\addplot[only marks, mark=*, blue] coordinates {(16.869047164916992, 18.130952835083008)};
\addplot[only marks, mark=*, green] coordinates {(16.452381134033203, 18.071428298950195)};
\addplot[only marks, mark=*, red] coordinates {(15.25, 16.380952835083008)};
\addplot[only marks, mark=*, yellow] coordinates {(14.773809432983398, 15.904762268066406)};
\addplot[only marks, mark=*, magenta] coordinates {(15.309523582458496, 16.0)};
\addplot[only marks, mark=*, cyan] coordinates {(14.452381134033203, 15.916666984558105)};
\addplot[only marks, mark=*, black] coordinates {(15.142857551574707, 15.4880952835083)};
\addplot[only marks, mark=*, olive] coordinates {(14.333333015441895, 15.880952835083008)};
\addplot[only marks, mark=*, teal] coordinates {(18.404762268066406, 18.452381134033203)};
\addplot[only marks, mark=*, purple] coordinates {(17.91666603088379, 18.678571701049805)};
;

\legend{airplane, automobile, bird, cat, deer, dog, frog, horse, ship, truck}

\end{axis}
\end{tikzpicture}
  \caption{Comparison of the number of pairs separated by each class. The x-axis represents the results from the Pretrained Resnet-50, while the y-axis corresponds to the Resnet18 trained from scratch on 4 classes of CIFAR10. The analysis highlights the most promising classes, such as truck, ship, airplane, and automobile.}
\end{subfigure}
  \caption{Separability given by a Resnet50: Training from scratch on subsets of classes of CIFAR10.}
\end{figure*}

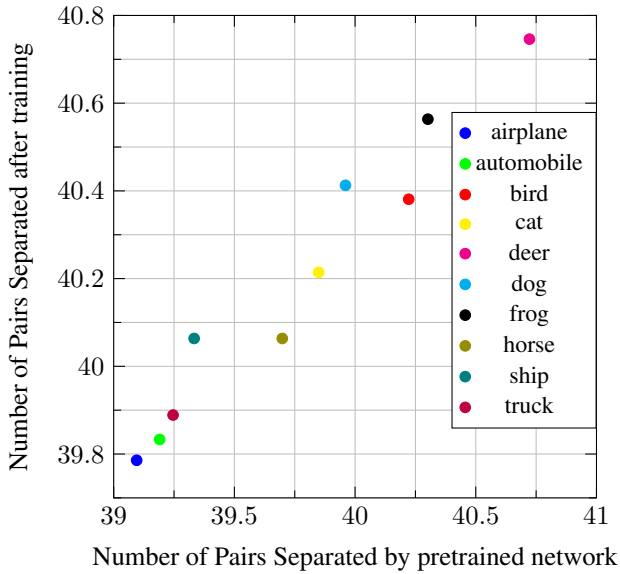
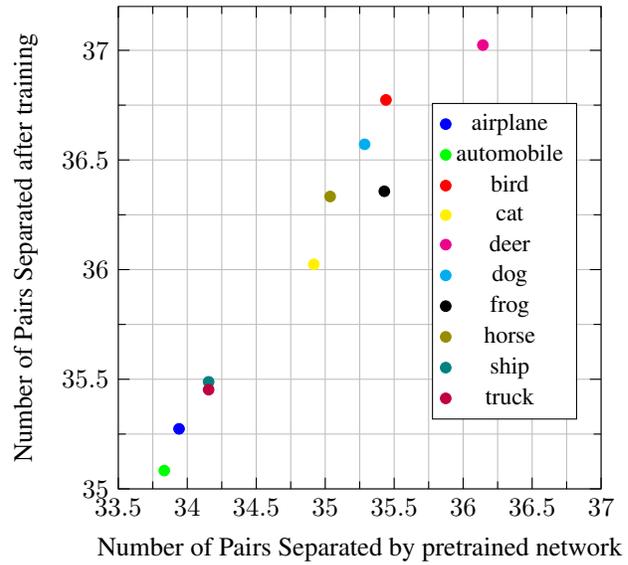
\begin{figure*}[h]
\centering
\begin{subfigure}{\columnwidth}
  \centering
\begin{tikzpicture}
\begin{axis}[
  xlabel={Number of Pairs Separated by pretrained network},
  ylabel={Number of Pairs Separated after training },
  xmin=39,  
  xmax=41,  
  ymin=39.7,  
  ymax=40.8,   
  minor x tick num=1, 
  minor y tick num=1, 
  grid=both,
  height=8cm,
  width = 8cm,
  every tick/.style={black},
  clip = false,
  clip mode=individual,
  nodes near coords=\empty, 
  visualization depends on={y \as \labelY},
  nodes near coords style={anchor=center, font=\tiny, inner sep=2pt},
  legend style={font=\small,    
  at={(0.85, 0.8)}, 
   anchor=north},
]

\addplot[only marks, mark=*, blue] coordinates {(39.095237731933594, 39.78571319580078)};
\addplot[only marks, mark=*, green] coordinates {(39.19047546386719, 39.83333206176758)};
\addplot[only marks, mark=*, red] coordinates {(40.22222137451172, 40.380950927734375)};
\addplot[only marks, mark=*, yellow] coordinates {(39.849205017089844, 40.21428680419922)};
\addplot[only marks, mark=*, magenta] coordinates {(40.72222137451172, 40.74603271484375)};
\addplot[only marks, mark=*, cyan] coordinates {(39.9603157043457, 40.412696838378906)};
\addplot[only marks, mark=*, black] coordinates {(40.30158615112305, 40.56349182128906)};
\addplot[only marks, mark=*, olive] coordinates {(39.69841384887695, 40.06349182128906)};
\addplot[only marks, mark=*, teal] coordinates {(39.33333206176758, 40.06349182128906)};
\addplot[only marks, mark=*, purple] coordinates {(39.24603271484375, 39.88888931274414)};

\legend{airplane, automobile, bird, cat, deer, dog, frog, horse, ship, truck}

\end{axis}
\end{tikzpicture}
  \caption{Comparison of the number of pairs separated by each class. The x-axis represents the results from the Pretrained VIT-8, while the y-axis corresponds to the VIT-8 trained from scratch on 6 classes of CIFAR10. The analysis highlights the most promising classes, such as deer, bird, dog and frog.}

\end{subfigure}%
\hfill
\begin{subfigure}{\columnwidth}
  \centering
\begin{tikzpicture}
\begin{axis}[
  xlabel={Number of Pairs Separated by pretrained network},
  ylabel={Number of Pairs Separated after training },
  xmin=33.5,  
  xmax=37,  
  ymin=35,  
  ymax=37.2,   
  minor x tick num=1, 
  minor y tick num=1, 
  grid=both,
  height=8cm,
  width = 8cm,
  every tick/.style={black},
  clip = false,
  clip mode=individual,
  nodes near coords=\empty, 
  visualization depends on={y \as \labelY},
  nodes near coords style={anchor=center, font=\tiny, inner sep=2pt},
  legend style={font=\small,    
  at={(0.8, 0.8)}, 
   anchor=north},
]

\addplot[only marks, mark=*, blue] coordinates {(33.94047546386719, 35.27381134033203)};
\addplot[only marks, mark=*, green] coordinates {(33.83333206176758, 35.08333206176758)};
\addplot[only marks, mark=*, red] coordinates {(35.44047546386719, 36.77381134033203)};
\addplot[only marks, mark=*, yellow] coordinates {(34.91666793823242, 36.02381134033203)};
\addplot[only marks, mark=*, magenta] coordinates {(36.14285659790039, 37.02381134033203)};
\addplot[only marks, mark=*, cyan] coordinates {(35.28571319580078, 36.57143020629883)};
\addplot[only marks, mark=*, black] coordinates {(35.42856979370117, 36.35714340209961)};
\addplot[only marks, mark=*, olive] coordinates {(35.03571319580078, 36.33333206176758)};
\addplot[only marks, mark=*, teal] coordinates {(34.154762268066406, 35.488094329833984)};
\addplot[only marks, mark=*, purple] coordinates {(34.154762268066406, 35.4523811340332)};

\legend{airplane, automobile, bird, cat, deer, dog, frog, horse, ship, truck}

\end{axis}
\end{tikzpicture}
  \caption{Comparison of the number of pairs separated by each class. The x-axis represents the results from the Pretrained VIT-8, while the y-axis corresponds to the VIT-8 trained from scratch on 4 classes of CIFAR10. The analysis highlights the most promising classes, such as deer, frog, horse, and bird.}
\end{subfigure}
  \caption{Separability given by a VIT-8: Training from scratch on subsets of classes of CIFAR10.}
\end{figure*}

\begin{figure*}[h]
\centering
\begin{subfigure}{\columnwidth}
  \centering
\begin{tikzpicture}
\begin{axis}[
  xlabel={Number of Pairs Separated by pretrained network},
  ylabel={Number of Pairs Separated after training },
  xmin=34.5,  
  xmax=37.5,  
  ymin=33.5,  
  ymax=36.5,   
  minor x tick num=1, 
  minor y tick num=1, 
  grid=both,
  height=8cm,
  width = 8cm,
  every tick/.style={black},
  clip = false,
  clip mode=individual,
  nodes near coords=\empty, 
  visualization depends on={y \as \labelY},
  nodes near coords style={anchor=center, font=\tiny, inner sep=2pt},
  legend style={font=\small,    
  at={(0.85, 0.75)}, 
   anchor=north},
]

\addplot[only marks, mark=*, blue] coordinates {(36.119049072265625, 35.57143020629883)};
\addplot[only marks, mark=*, green] coordinates {(36.345237731933594, 35.82143020629883)};
\addplot[only marks, mark=*, red] coordinates {(36.55952453613281, 36.07143020629883)};
\addplot[only marks, mark=*, yellow] coordinates {(37.17856979370117, 36.46428680419922)};
\addplot[only marks, mark=*, magenta] coordinates {(36.69047546386719, 35.904762268066406)};
\addplot[only marks, mark=*, cyan] coordinates {(34.72618865966797, 34.53571319580078)};
\addplot[only marks, mark=*, black] coordinates {(36.095237731933594, 35.14285659790039)};
\addplot[only marks, mark=*, olive] coordinates {(34.5476188659668, 33.845237731933594)};
\addplot[only marks, mark=*, teal] coordinates {(34.85714340209961, 34.33333206176758)};
\addplot[only marks, mark=*, purple] coordinates {(34.92856979370117, 34.5476188659668)};

\legend{T-shirt/top, Trouser, Pullover, Dress, Coat, Sandal, Shirt, Sneaker, Bag, Ankle boot}
\end{axis}
\end{tikzpicture}
  \caption{Comparison of the number of pairs separated by each class. The x-axis represents the results from the Pretrained Resnet-50, while the y-axis corresponds to the Resnet-50 trained from scratch on 4 classes of FASHION-MNIST. The analysis highlights the most promising classes, such as Dress, Pullover, Coat.}

\end{subfigure}%
\hfill
\begin{subfigure}{\columnwidth}
  \centering
\begin{tikzpicture}
\begin{axis}[
  xlabel={Number of Pairs Separated by pretrained network},
  ylabel={Number of Pairs Separated after training },
  xmin=34,  
  xmax=37,  
  ymin=34.5,  
  ymax=38,   
  minor x tick num=1, 
  minor y tick num=1, 
  grid=both,
  height=8cm,
  width = 8cm,
  every tick/.style={black},
  clip = false,
  clip mode=individual,
  nodes near coords=\empty, 
  visualization depends on={y \as \labelY},
  nodes near coords style={anchor=center, font=\tiny, inner sep=2pt},
  legend style={font=\small,    
  at={(0.85, 0.65)}, 
   anchor=north},
]
\addplot[only marks, mark=*, blue] coordinates {(36.41666793823242, 37.011905670166016)};
\addplot[only marks, mark=*, green] coordinates {(36.55952453613281, 37.0)};
\addplot[only marks, mark=*, red] coordinates {(36.7023811340332, 37.58333206176758)};
\addplot[only marks, mark=*, yellow] coordinates {(36.91666793823242, 37.619049072265625)};
\addplot[only marks, mark=*, magenta] coordinates {(36.71428680419922, 37.28571319580078)};
\addplot[only marks, mark=*, cyan] coordinates {(34.7023811340332, 35.630950927734375)};
\addplot[only marks, mark=*, black] coordinates {(35.46428680419922, 36.130950927734375)};
\addplot[only marks, mark=*, olive] coordinates {(34.82143020629883, 34.761905670166016)};
\addplot[only marks, mark=*, teal] coordinates {(34.630950927734375, 35.369049072265625)};
\addplot[only marks, mark=*, purple] coordinates {(35.404762268066406, 35.654762268066406)};
;

\legend{T-shirt/top, Trouser, Pullover, Dress, Coat, Sandal, Shirt, Sneaker, Bag, Ankle boot}
\end{axis}
\end{tikzpicture}
  \caption{Comparison of the number of pairs separated by each class. The x-axis represents the results from the Pretrained VIT-8, while the y-axis corresponds to the VIT-8 trained from scratch on 4 classes of FASHION-MNIST. The analysis highlights the most promising classes, such as Dress, Pullover, Coat, Trouser and T-shirt.}

\end{subfigure}
  \caption{Training from scratch on subsets of classes of FASHION-MNIST}
\label{fig:sepclassesfinetuning}
\end{figure*}
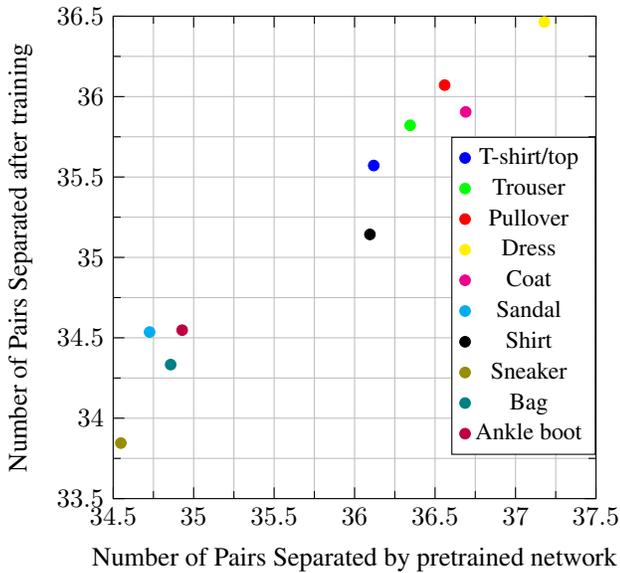
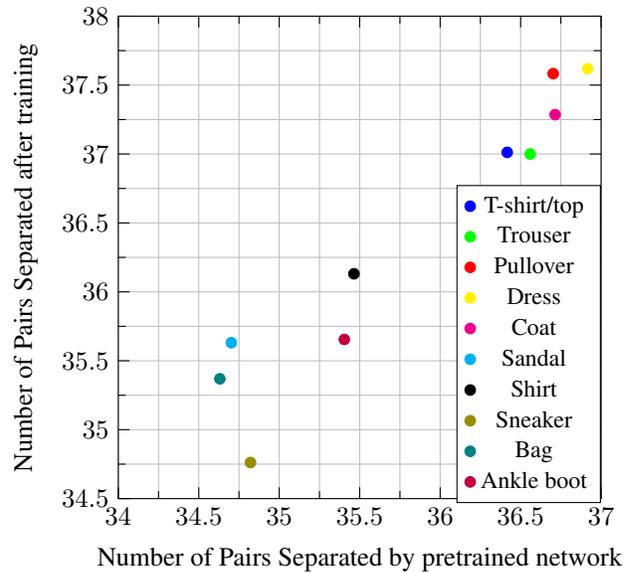

\begin{figure*}[h]
\centering
\begin{subfigure}{\columnwidth}
  \centering
\begin{tikzpicture}
\begin{axis}[
  xlabel={Number of Pairs Separated by pretrained network},
  ylabel={Number of Pairs Separated after training },
  xmin=28.5  ,  
  xmax=31.5,  
  ymin=33,  
  ymax=36,   
  minor x tick num=1, 
  minor y tick num=1, 
  grid=both,
  height=8cm,
  width = 8cm,
  every tick/.style={black},
  clip = false,
  clip mode=individual,
  nodes near coords=\empty, 
  visualization depends on={y \as \labelY},
  nodes near coords style={anchor=center, font=\tiny, inner sep=2pt},
  legend style={font=\small,    
  at={(0.8, 0.8)}, 
   anchor=north},
]

\addplot[only marks, mark=*, blue] coordinates {(29.14285659790039, 33.75)};
\addplot[only marks, mark=*, green] coordinates {(28.619047164916992, 33.41666793823242)};
\addplot[only marks, mark=*, red] coordinates {(29.547618865966797, 34.53571319580078)};
\addplot[only marks, mark=*, yellow] coordinates {(29.91666603088379, 35.66666793823242)};
\addplot[only marks, mark=*, magenta] coordinates {(30.238094329833984, 34.738094329833984)};
\addplot[only marks, mark=*, cyan] coordinates {(29.64285659790039, 33.78571319580078)};
\addplot[only marks, mark=*, black] coordinates {(29.10714340209961, 34.7023811340332)};
\addplot[only marks, mark=*, olive] coordinates {(31.2261905670166, 35.261905670166016)};
\addplot[only marks, mark=*, teal] coordinates {(29.369047164916992, 34.16666793823242)};
\addplot[only marks, mark=*, purple] coordinates {(30.0, 34.261905670166016)};

\legend{0, 1, 2, 3, 4, 5, 6, 7, 8, 9}
\end{axis}
\end{tikzpicture}
  \caption{Comparison of the number of pairs separated by each class. The x-axis represents the results from the Pretrained Resnet-50, while the y-axis corresponds to the Resnet-50 trained from scratch on 4 classes of MNIST. The analysis highlights the most promising classes, such as $7$ and $4$}
\end{subfigure}%
\hfill
\begin{subfigure}{\columnwidth}
  \centering
\begin{tikzpicture}
\begin{axis}[
  xlabel={Number of Pairs Separated by pretrained network},
  ylabel={Number of Pairs Separated after training },
  xmin=37,  
  xmax=39,  
  ymin=38.2,  
  ymax=39.4,   
  minor x tick num=1, 
  minor y tick num=1, 
  grid=both,
  height=8cm,
  width = 8cm,
  every tick/.style={black},
  clip = false,
  clip mode=individual,
  nodes near coords=\empty, 
  visualization depends on={y \as \labelY},
  nodes near coords style={anchor=center, font=\tiny, inner sep=2pt},
  legend style={font=\small,    
  at={(0.2, 0.8)}, 
   anchor=north},
]

\addplot[only marks, mark=*, blue] coordinates {(37.988094329833984, 38.880950927734375)};
\addplot[only marks, mark=*, green] coordinates {(37.02381134033203, 38.238094329833984)};
\addplot[only marks, mark=*, red] coordinates {(38.33333206176758, 38.869049072265625)};
\addplot[only marks, mark=*, yellow] coordinates {(38.44047546386719, 39.21428680419922)};
\addplot[only marks, mark=*, magenta] coordinates {(38.07143020629883, 39.0476188659668)};
\addplot[only marks, mark=*, cyan] coordinates {(37.92856979370117, 38.72618865966797)};
\addplot[only marks, mark=*, black] coordinates {(38.7976188659668, 39.2023811340332)};
\addplot[only marks, mark=*, olive] coordinates {(38.58333206176758, 39.27381134033203)};
\addplot[only marks, mark=*, teal] coordinates {(37.988094329833984, 38.7976188659668)};
\addplot[only marks, mark=*, purple] coordinates {(37.7023811340332, 38.654762268066406)}; 

\legend{0, 1, 2, 3, 4, 5, 6, 7, 8, 9}
\end{axis}
\end{tikzpicture}
  \caption{Comparison of the number of pairs separated by each class. The x-axis represents the results from the Pretrained VIT-8, while the y-axis corresponds to the VIT-8 trained from scratch on 4 classes of MNIST. The analysis highlights the most promising classes, such $6$, $7$ and $3$}

\end{subfigure}
  \caption{Training from scratch on subsets of classes of MNIST}
\end{figure*}
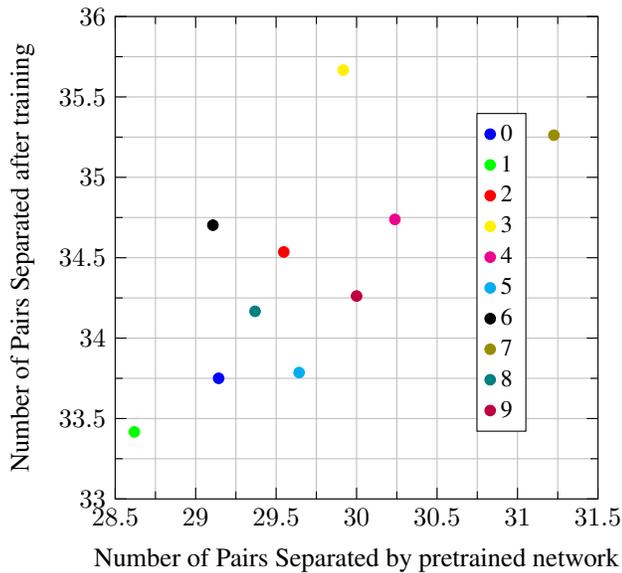
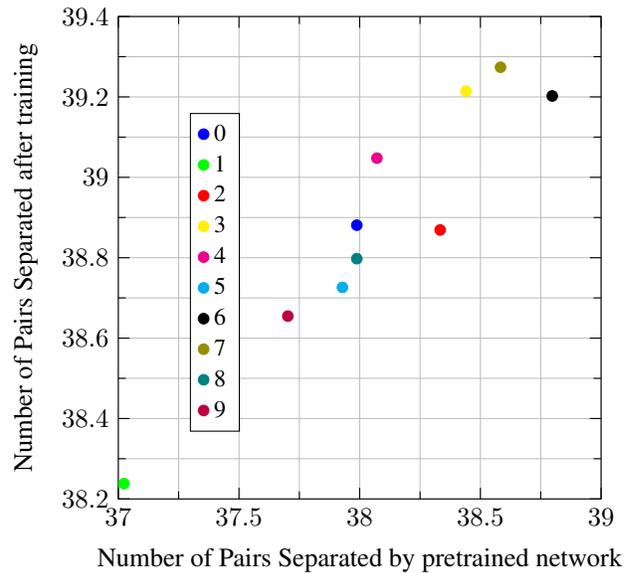

\begin{figure*}[h]
\centering
 \begin{subfigure}{0.45\textwidth}
  \centering
  \begin{tikzpicture}
   \begin{axis}[ xlabel={Number of Pairs Separated by pretrained network},
  ylabel={Number of Pairs Separated after training },    grid=both,height=7cm,]
    \addplot[blue,only marks,mark=*,    ] table [x=x, y=y, col sep=comma]    {CSV/separability_cifar10_resnet50_4_classes.csv};
    \addplot[smooth,red,] table [x=x, y={create col/linear regression={y=y}}, col sep=comma] {CSV/separability_cifar10_resnet50_4_classes.csv};
   \end{axis}
  \end{tikzpicture}
  \caption{Number of Pairs Separated by pretrained network vs after training from scratch on subsets of 4 classes of CIFAR10. X-axis corresponds to the number of pairs separated by the Pretrained Resnet50.}
  \label{fig:fig1}
 \end{subfigure}%
 \hfill
 \begin{subfigure}{0.45\textwidth}
  \centering
  \begin{tikzpicture}
   \begin{axis}[ xlabel={Number of Pairs Separated by pretrained network},
  ylabel={Number of Pairs Separated after training },    grid=both,height=7cm,]
    \addplot[blue,only marks,mark=*,    ] table [x=x, y=y, col sep=comma]    {CSV/separability_cifar10_resnet50_6_classes.csv};
    \addplot[smooth,red,] table [x=x, y={create col/linear regression={y=y}}, col sep=comma] {CSV/separability_cifar10_resnet50_6_classes.csv};
   \end{axis}
  \end{tikzpicture}
  \caption{Number of Pairs Separated by pretrained network vs after training from scratch on subsets of 6 classes of CIFAR10. X-axis corresponds to the number of pairs separated by the Pretrained Resnet50.}
  \label{fig:fig2}
 \end{subfigure}
 \caption{Training from Scratch on subsets of CIFAR10, Resnet-50 employed to identify promising subsets.}
 \label{fig:overall}
\end{figure*}
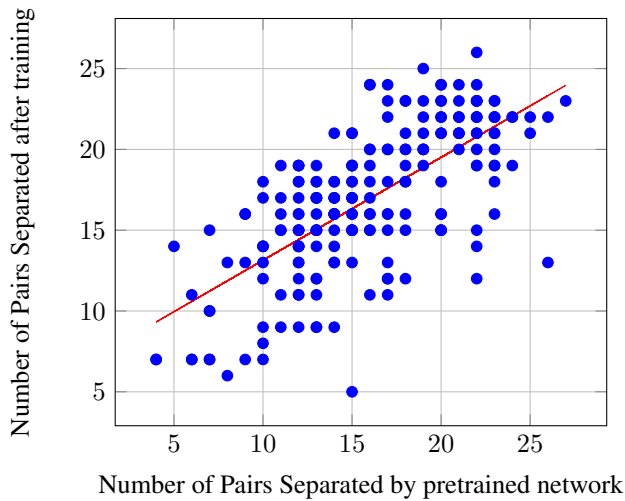
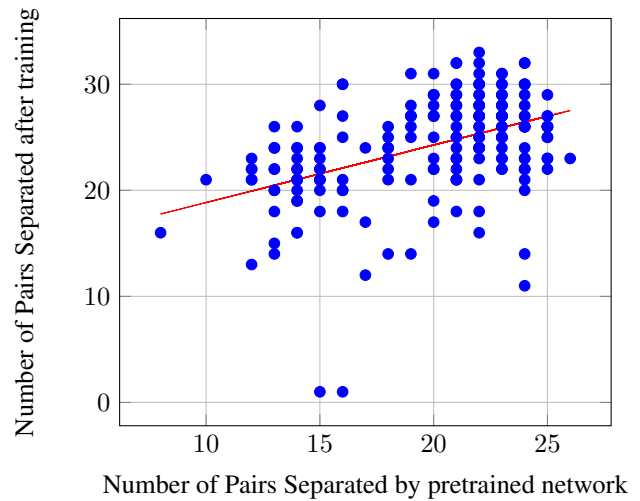

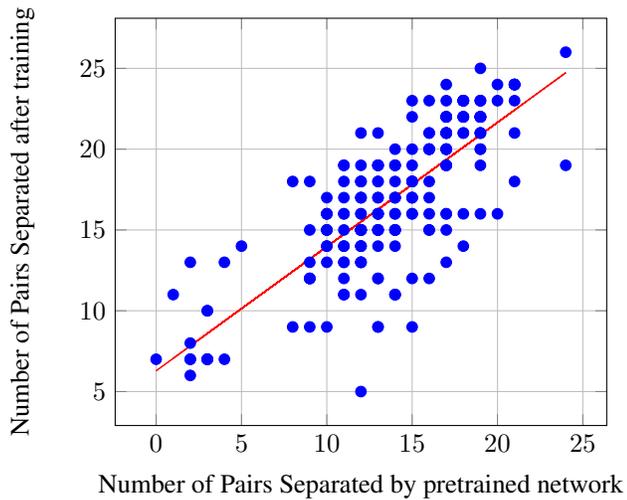
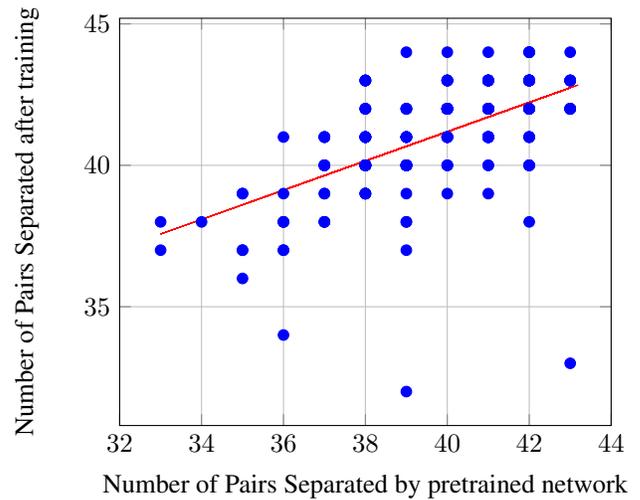
\begin{figure*}[h]
\centering
 \begin{subfigure}{0.45\textwidth}
  \centering
  \begin{tikzpicture}
   \begin{axis}[ xlabel={Number of Pairs Separated by pretrained network},
  ylabel={Number of Pairs Separated after training },grid=both,height=7cm,]
    \addplot[blue,only marks,mark=*,    ] table [x=x, y=y, col sep=comma]    {CSV/separability_cifar10_vit8_4_classes.csv};
    \addplot[smooth,red,] table [x=x, y={create col/linear regression={y=y}}, col sep=comma] {CSV/separability_cifar10_vit8_4_classes.csv};
   \end{axis}
  \end{tikzpicture}
  \caption{Number of Pairs Separated by pretrained network vs after training from scratch on subsets of 4 classes of CIFAR10. X-axis corresponds to the number of pairs separated by the Pretrained VIT-8.}
  \label{fig:fig1}
 \end{subfigure}%
 \hfill
 \begin{subfigure}{0.45\textwidth}
  \centering
  \begin{tikzpicture}
   \begin{axis}[ xlabel={Number of Pairs Separated by pretrained network},
  ylabel={Number of Pairs Separated after training } ,  grid=both,height=7cm,]
    \addplot[blue,only marks,mark=*,    ] table [x=x, y=y, col sep=comma]    {CSV/separability_cifar10_vit8_6_classes.csv};
    \addplot[smooth,red,] table [x=x, y={create col/linear regression={y=y}}, col sep=comma] {CSV/separability_cifar10_vit8_6_classes.csv};
   \end{axis}
  \end{tikzpicture}
  \caption{Number of Pairs Separated by pretrained network vs after training from scratch on subsets of 6 classes of CIFAR10. X-axis corresponds to the number of pairs separated by the Pretrained VIT-8.}
  \label{fig:fig2}
 \end{subfigure}
 \caption{Training from Scratch on subsets of CIFAR10, VIT-8 employed to identify promising subsets.}
 \label{fig:overall}
\end{figure*}

\begin{figure*}[h]
\centering
 \begin{subfigure}{0.45\textwidth}
  \centering
  \begin{tikzpicture}
   \begin{axis}[ xlabel={Number of Pairs Separated by pretrained network},
  ylabel={Number of Pairs Separated after training },grid=both,height=7cm,]
    \addplot[blue,only marks,mark=*,    ] table [x=x, y=y, col sep=comma]    {CSV/separability_fashionmnist_resnet50_4_classes.csv};
    \addplot[smooth,red,] table [x=x, y={create col/linear regression={y=y}}, col sep=comma] {CSV/separability_fashionmnist_resnet50_4_classes.csv};
   \end{axis}
  \end{tikzpicture}
  \caption{Number of Pairs Separated by pretrained network vs after training from scratch. X-axis corresponds to the number of pairs separated by the Pretrained Resnet50.}
  \label{fig:fig1}
 \end{subfigure}%
 \hfill
 \begin{subfigure}{0.45\textwidth}
  \centering
  \begin{tikzpicture}
   \begin{axis}[ xlabel={Number of Pairs Separated by pretrained network},
  ylabel={Number of Pairs Separated after training } ,  grid=both,height=7cm,]
    \addplot[blue,only marks,mark=*,    ] table [x=x, y=y, col sep=comma]    {CSV/separability_fashionmnist_vit8_4_classes.csv};
    \addplot[smooth,red,] table [x=x, y={create col/linear regression={y=y}}, col sep=comma] {CSV/separability_fashionmnist_vit8_4_classes.csv};
   \end{axis}
  \end{tikzpicture}
  \caption{Number of Pairs Separated by pretrained network vs after training from scratch. X-axis corresponds to the number of pairs separated by the Pretrained VIT-8.}
  \label{fig:fig2}
 \end{subfigure}
 \caption{Training from Scratch on subsets composed of 4 classes of FASHION-MNIST.}
 \label{fig:overall}
\end{figure*}

\begin{figure*}[h]
\centering
 \begin{subfigure}{0.45\textwidth}
  \centering
  \begin{tikzpicture}
   \begin{axis}[ xlabel={Number of Pairs Separated by pretrained network},
  ylabel={Number of Pairs Separated after training },grid=both,height=7cm,]
    \addplot[blue,only marks,mark=*,    ] table [x=x, y=y, col sep=comma]    {CSV/separability_mnist_resnet50_4_classes.csv};
    \addplot[smooth,red,] table [x=x, y={create col/linear regression={y=y}}, col sep=comma] {CSV/separability_mnist_resnet50_4_classes.csv};
   \end{axis}
  \end{tikzpicture}
  \caption{Number of Pairs Separated by pretrained network vs after training from scratch. X-axis corresponds to the number of pairs separated by the Pretrained Resnet50.}
  \label{fig:fig1}
 \end{subfigure}%
 \hfill
 \begin{subfigure}{0.45\textwidth}
  \centering
  \begin{tikzpicture}
   \begin{axis}[ xlabel={Number of Pairs Separated by pretrained network},
  ylabel={Number of Pairs Separated after training } ,  grid=both,height=7cm,]
    \addplot[blue,only marks,mark=*,    ] table [x=x, y=y, col sep=comma]    {CSV/separability_mnist_vit8_4_classes.csv};
    \addplot[smooth,red,] table [x=x, y={create col/linear regression={y=y}}, col sep=comma] {CSV/separability_mnist_vit8_4_classes.csv};
   \end{axis}
  \end{tikzpicture}
  \caption{Number of Pairs Separated by pretrained network vs after training from scratch. X-axis corresponds to the number of pairs separated by the Pretrained VIT-8.}
  \label{fig:fig2}
 \end{subfigure}
 \caption{Training from Scratch on subsets composed of 4 classes of MNIST.}
 \label{fig:overall}
\end{figure*}

\begin{figure*}
    \centering
    \begin{subfigure}{0.5\textwidth}
        \centering
        \includegraphics[width=0.8\linewidth, height=0.8\linewidth, keepaspectratio]{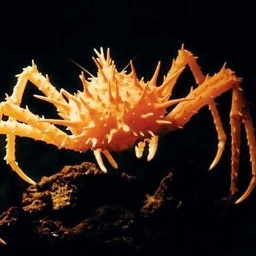}
        \caption{Crab}
    \end{subfigure}%
    \begin{subfigure}{0.5\textwidth}
        \centering
        \includegraphics[width=0.8\linewidth, height=0.8\linewidth, keepaspectratio]{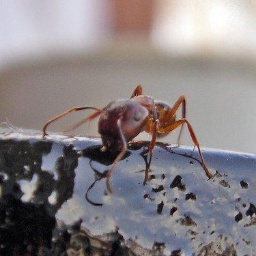}
        \caption{Ants}
    \end{subfigure}
    
    \begin{subfigure}{0.5\textwidth}
        \centering
        \includegraphics[width=0.8\linewidth, height=0.8\linewidth, keepaspectratio]{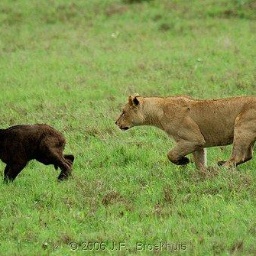}
        \caption{Lion}
    \end{subfigure}%
    \begin{subfigure}{0.5\textwidth}
        \centering
        \includegraphics[width=0.8\linewidth, height=0.8\linewidth, keepaspectratio]{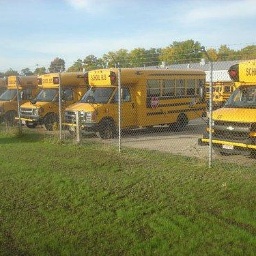}
        \caption{Bus}
    \end{subfigure}
    
    \begin{subfigure}{0.5\textwidth}
        \centering
        \includegraphics[width=0.8\linewidth, height=0.8\linewidth, keepaspectratio]{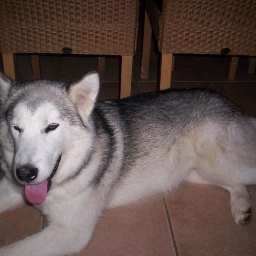}
        \caption{Husky}
    \end{subfigure}%
    \begin{subfigure}{0.5\textwidth}
        \centering
        \includegraphics[width=0.8\linewidth, height=0.8\linewidth, keepaspectratio]{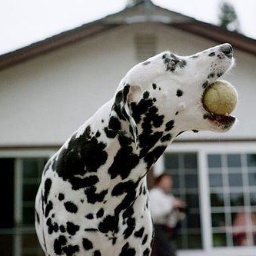}
        \caption{Dalmatian}
    \end{subfigure}
    
\caption{Depiction of the most challenging pairs on TIEREDIMAGENET: Crab vs Ants;  Lion vs Bus; Husky vs Dalmatian.}
    \label{fig:enter-label}
\end{figure*}